\documentclass{article}

\usepackage{authblk,textcomp}
\usepackage{wrapfig}
\usepackage{enumitem}
\usepackage{caption}
\usepackage{hyperref}       
\usepackage{url}            
\usepackage[cal=euler]{mathalfa}
\usepackage{libertine}
\usepackage{mathtools}
\usepackage{parskip}
\usepackage{graphicx}
\usepackage{subfigure}  
\usepackage{mathtools}
\usepackage{bbm}
\usepackage[capitalize,noabbrev]{cleveref}
\usepackage{commath}
\usepackage{amsmath, amssymb, amsthm, amsfonts}       
\usepackage{xcolor}         
\usepackage{natbib}[numbers, sort&compress]
\usepackage{geometry}
\usepackage{times}

\renewcommand{\vec}[1]{\bm{#1}}

\hypersetup{pdfauthor={IdePHICS},pdftitle={RepetitaIuvant},%
            colorlinks, linktocpage=true, pdfstartpage=1, pdfstartview=FitV,%
    breaklinks=true, pdfpagemode=UseNone, pageanchor=true, pdfpagemode=UseOutlines,%
    plainpages=false, bookmarksnumbered, bookmarksopen=true, bookmarksopenlevel=1,%
    hypertexnames=true, pdfhighlight=/O,%
    urlcolor=orange, linkcolor=blue, citecolor=blue
        }

\newtheorem*{algorithm*}{Algorithm}

\newtheorem{assumption}{Assumption}
\newtheorem{myalgorithm*}{Algorithm}
\newtheorem{myalgorithm}{Algorithm}
\newtheorem{theorem}{Theorem}
\newtheorem{lemma}{Lemma}
\newtheorem{proposition}{Proposition}
\newtheorem{definition}{Definition}
\newtheorem{corollary}{Corollary}

\makeatletter
\newcommand{\otherlabel}[2]{\protected@edef\@currentlabel{#2}\label{#1}}
\makeatother

\newcommand{\dP}{\mathbb{P}}

\newcommand{\Ea}[1]{\E\left[#1\right]}
\newcommand{\Eb}[2]{\E_{#1}\left[#2\right]}

\renewcommand{\vec}[1]{\boldsymbol{#1}}

  \providecommand{\R}{\mathbb{R}} 
  \providecommand{\N}{\mathbb{N}} 

  \makeatletter
  \def\sign{\@ifnextchar*{\@sgnargscaled}{\@ifnextchar[{\sgnargscaleas}{\@ifnextchar{\bgroup}{\@sgnarg}{\sgn} }}}
  \def\@sgnarg#1{\sgn\rbr{#1}}
  \def\@sgnargscaled#1{\sgn\rbr*{#1}}
  \def\@sgnargscaleas[#1]#2{\sgn\rbr[#1]{#2}}
  \makeatother

\providecommand{\E}{\mathbb{E}}%

  \providecommand{\cL}{\mathcal{L}}
  
  \providecommand{\cN}{\mathcal{N}}
  
  \providecommand{\cP}{\mathcal{P}}
  
  \providecommand{\cR}{\mathcal{R}}


\newcommand{\speedup}[1]{{\color{gray}(\ifdim #1 pt > 0.3pt #1\else $< #1$\fi{}$\times$)}}

\makeatletter
\newsavebox{\@brx}
\newcommand{\llangle}[1][]{\savebox{\@brx}{\(\m@th{#1\langle}\)}%
  \mathopen{\copy\@brx\mkern2mu\kern-0.9\wd\@brx\usebox{\@brx}}}
\newcommand{\rrangle}[1][]{\savebox{\@brx}{\(\m@th{#1\rangle}\)}%
  \mathclose{\copy\@brx\mkern2mu\kern-0.9\wd\@brx\usebox{\@brx}}}
\makeatother

\providecommand{\norm}[1]{\left\lVert#1\right\rVert}

  \providecommand{\R}{\mathbb{R}} %
  \providecommand{\N}{\mathbb{N}} %

  \providecommand{\Eb}[1]{{\mathbb E}\left[#1\right] }

  \providecommand{\cL}{\mathcal{L}}
  
  \providecommand{\cN}{\mathcal{N}}
  
  \providecommand{\cP}{\mathcal{P}}
  
  \providecommand{\cR}{\mathcal{R}}

\renewcommand{\vec}{\mathbf}

\providecommand{\mycomment}[3]{\todo[caption={},size=footnotesize,color=#1!20]{\textbf{#2: }#3}}%
\providecommand{\inlinecomment}[3]{%
  {\color{#1}#2: #3}}%
\newcommand\commenter[2]%
{%
  \expandafter\newcommand\csname i#1\endcsname[1]{\inlinecomment{#2}{#1}{##1}}
  \expandafter\newcommand\csname #1\endcsname[1]{\mycomment{#2}{#1}{##1}}
}

\title{Repetita Iuvant: Data Repetition Allows SGD to Learn High-Dimensional Multi-Index Functions}

\author[1]{Luca Arnaboldi}
\author[1,2]{Yatin Dandi}
\author[1]{Florent Krzakala}
\author[1]{Luca Pesce}
\author[3]{Ludovic Stephan}

\affil[1]{\small Ecole Polytechnique F\'{e}d\'{e}rale de Lausanne, 
Information, Learning and Physics Laboratory. CH-1015 Lausanne, Switzerland.}
\affil[2]{\small 
Ecole Polytechnique F\'{e}d\'{e}rale de Lausanne,
Statistical Physics of Computation Laboratory. CH-1015 Lausanne, Switzerland.}
\affil[3]{\small Univ Rennes, Ensai, CNRS, CREST – UMR 9194
  F-35000 Rennes, France}

\date{}
\begin{document}

\maketitle

\begin{abstract} 
  Neural networks can identify low-dimensional relevant structures within high-dimensional noisy data, yet our mathematical understanding of how they do so remains scarce. Here, we investigate the training dynamics of two-layer shallow neural networks trained with gradient-based algorithms, and discuss how they learn pertinent features in multi-index models, that is target functions with low-dimensional relevant directions. In the high-dimensional regime, where the input dimension $d$ diverges, we show that a simple modification of the idealized single-pass gradient descent training scenario, where data can now be repeated or iterated upon twice, drastically improves its computational efficiency. In particular, it surpasses the limitations previously believed to be dictated by the Information and Leap exponents associated with the target function to be learned. Our results highlight the ability of networks to learn relevant structures from data alone without any pre-processing. More precisely, we show that (almost) all directions are learned with at most $O(d \log d)$ steps. Among the exceptions is a set of hard functions that includes sparse parities. In the presence of coupling between directions, however, these can be learned sequentially through a hierarchical mechanism that generalizes the notion of staircase functions. Our results are proven by a rigorous study of the evolution of the relevant statistics for high-dimensional dynamics. 
\end{abstract}

\section{Introduction}
Gradient Descent-based algorithms such as  Stochastic Gradient Descent (SGD) and its variations, are a fundamental tool in training neural networks, and are therefore the subject of intense theoretical scrutiny. Recent years have witnessed significant advancements in understanding their dynamics and the learning mechanisms of neural nets. Significant progress has been made, in particular, in the case of two-layer networks thanks, in part, to the so-called mean-field analysis \citep{mei2018mean,chizat2018global,rotskoff2018trainability,sirignano2020mean}). A large part of the theoretical approach focused on \emph{one-pass} optimization algorithms, where \emph{each iteration involves a new fresh batch of data}. In particular, in the mathematically solvable case of  high-dimensional synthetic Gaussian data, and a low dimensional a multi-index target function model, the class of functions efficiently learned by these \textit{one-pass} methods has been thoroughly analyzed in a series of recent works, and have been shown to be limited by the so-called information exponent \citep{BenArous2021} and leap exponent \citep{abbe2022merged, abbe2023sgd}. These analyses have sparked many follow-up theoretical works over the last few months, see, e.g. \cite{damian2022neural,damian_2023_smoothing,dandi2023twolayer,bietti2023learning,ba2023learning, moniri2023theory, mousavihosseini2023gradientbased,zweig2023symmetric}, leading to a picture where SGD dynamics is governed by the information exponent, that is, the order of the first non-zero Hermite coefficient of the target.

A common point between all these works, however, is that they considered the idealized situation where one process data one sample at a time, all of them independently identically distributed, without any repetition. This is, however, not a realistic situation in practice. Indeed, most datasets contains similar data-points, so that repetition occurs (see e.g. for repetition and duplicate in CIFAR \citep{barz2020we}) Additionally, it is common in machine learning to  repeatedly go through the same mini-batch of data multiple time over epoch. Finally, many gradient algorithms, often used in machine learning in Lookahead optimizers \citep{lookahead}, Extragradient methods \citep{Korpelevich1976TheEM}, or Sharpness Aware Minimization (SAM) \citep{foret2021sharpnessaware}, are \emph{explicitly} few gradient steps over the exact same data-point. It is thus natural to wonder if using such extra-gradient algorithms, or simply reusing many time some data, would change the picture with respect to the current consensus.

There are good reason to believe this to be the case. Indeed, it has been observed very recently \citep{dandi2024benefits} that iterating twice over large ($O(d)$) datasets was enough to alter the picture reached in 
\cite{BenArous2021,abbe2022merged, abbe2023sgd}. Indeed \cite{dandi2024benefits} showed that their exist functions that can be learned over few (just two) batch repetition, while they require a much larger (i.e. polynomial in $d$) number of steps otherwise.
While the set of function considered in this work was limited, the conclusion is indeed surprising, thus motivating the following question: 
\begin{center}
\textit{
Can data repetition increase the efficiency of SGD when learning any multi-index functions?
} 
\end{center}


We establish a positive answer to the above question by the analysis of more realistic optimization schemes with respect to the standard One-Pass SGD. Indeed, the hardness exponents developed in the seminal works \cite{abbe2021staircase, BenArous2021} are substantially bound to the idealized training scenario that considers i.i.d data. Slight modifications to this training scenario toward a more realistic setting starkly change the global picture. Such slight modifications include: a) non-vanishing correlations between different data points, that is a natural requirement for real datasets; b) processing the same data point multiple times in the optimization routine, i.e., a standard step in any algorithmic procedure when looping over different epochs. 

We consider a simple yet analytically tractable single-pass training scheme: processing one Gaussian sample at a time, but \emph{repeating the gradient step twice}. This minor adjustment, common in Lookahead optimizers \citep{lookahead}, Extragradient methods \citep{Korpelevich1976TheEM}, and Sharpness-Aware Minimization (SAM) \citep{foret2021sharpnessaware}, significantly boosts efficiency. In fact, we will see that this approach matches the performance of our best algorithms in most cases.

We show in particular that most multi-index function are learned with total sample complexity either $O(d)$ or $O(d \log d)$ if they are related to the presence of a symmetry, without any need for preprocessing of the data that is performed in various works (See e.g., \cite{mondelli2018fundamental,luo2019optimal,maillard2020phase,chen2020learning}).  Our results thus demonstrate that shallow neural networks turns out to be significantly more powerful than previously believed \citep{BenArous2021,abbe2022merged,abbe2023sgd,arnaboldi.stephan.ea_2023_high} in identifying relevant feature in the data. In fact, they are shown in many (but not all) cases to saturate the bounds predicted by statistical queries (SQ)  \citep{damian2024computational} (rather than the more limited correlation statistical queries (CSQ) bounds that were characteristic of the former consensus \citep{BenArous2021,abbe2023sgd}). Equivalently, it roughly  mirrors the prediction of the approximate message passing analysis in \citep{troiani2024fundamental}. Our conclusions follow from the rigorous analysis of the overlaps between the trained weight vectors with the low-dimensional relevant subspace that draws inspiration from the pivotal works of \cite{saad.solla_1995_line, PhysRevE.62.5444,A.C.C.Coolen_2000,BenArous2021}.

Concurrently with the present work, \citep{lee2024neural} also studies the properties of SGD under data repetition in neural networks. We discuss the differences between our works in Section~\ref{subsec:single_index_discussion}.

\section{Setting and Main Contributions}
\label{sec:main:setting}

Given a dataset of $N$ labeled examples $\{\vec z^\nu, y^\nu\}_{\nu \in [N]}$ in $d$ dimensions, we analyze the learning properties of fully connected two-layer networks with first layer weights $W=\{\vec w_j\}_{j \in [p]}$, second layer $\vec a \in \mathbb{R}^p$, and activation function $\sigma: \mathbb R \to \mathbb R$
\begin{align}
    f(\vec z; W, \vec a) = \frac{1}{p} \sum_{j =1}^p a_j \sigma \left(\langle \vec z, \vec w_j \rangle \right) \label{eq:def_nn}
\end{align}
We consider target functions that are dependent only on few orthogonal relevant directions $k=O_d(1)$ encoded in the target's weight matrix $W^\star = \{\vec w^\star_r\}_{r \in [k]} \in \mathbb{R}^{k\times d}$:
\begin{align}
    y^\nu = f^\star(\vec z^\nu) = h^\star(W^\star \vec z^\nu)
\end{align}
This model for structured data is usually referred to as a \textit{multi-index} model. Our main objective will be to analyze how efficiently two-layer nets adapt to this low-dimensional structure during training. To ensure that all necessary Gaussian expectations exist, we make the following non-restrictive assumption on $\sigma, h^\star$:
\begin{assumption}\label{assump:poly_growth}
    The functions $\sigma$ and $h^\star$ have at most polynomial growth, in the sense that there exists a constant $C > 0$ such that for all $x \in \mathbb{R}, \vec{x} \in \mathbb{R}^k$,
    \begin{equation}\label{eq:poly_growth}
        |\sigma(x)| \leq C(1 + |x|)^C \quad \text{and} \quad |h^\star(\vec{x})| \leq C(1 + \|\vec{x}\|)^C.
    \end{equation} 
\end{assumption}
\paragraph{Training algorithm --}
Our goal is to minimize the following objective:
\begin{equation}
    \cR(W, \vec{a}) = \E_{\vec{z}, y}\left[ \cL(f(\vec{z}, W, \vec{a}), y) \right]
\end{equation}
where $\cL: \R \times \R \to \R$ is a loss function, and the expectation is over the data distribution.

This objective is non-convex, and we do not have direct access to the expectation over the data. A central question in Machine Learning is therefore to quantify the properties of the weights attained at convergence by different algorithms. In this manuscript we consider a specific class of one-pass training routines to minimize the empirical risk. First, we partition the dataset in disjoint mini-batches of size $n_b$, i.e. $\mathcal{D} = \cup_{t \in [T]}\{\vec x^\nu, y^\nu\}_{\nu \in [tn_b, (t+1)n_b]}$ where $T$ is the total number of iterations considered. For each algorithmic step, we update the hidden layer neurons $W=\{\vec w_j\}_{j \in [p]}$ with the program $\mathcal{A}: (\vec w_{j,t}, \gamma, \rho, n_b) \in \mathbb{R}^{d+3} \to \vec w_{t+1,j} \in \mathbb{R}^d$.
\begin{myalgorithm}[Optimizer Main Step]
\label{algo:optimizer_main_step}
 \begin{equation}
 \begin{aligned}
     \vec w_{t+1,j} & =  \mathcal{A}(\vec w_t, \gamma, \rho, n_b) 
     =  \vec w_{t,j} - \frac{\gamma }{n_b}\sum_{\nu \in [n_b]
     }\nabla_{\vec w_{t,j}}
     \cL
    \left(
    f (\vec z^\nu; 
    \Tilde{W}_t(\rho), \vec a_0 ),  y^\nu  \right) \\
   \Tilde{\vec{w}}_{t, j} (\rho) &= \vec{w}_{t, j} - \rho \nabla_{\vec w_{t,j}}
    \cL\left( f(\vec z^\nu; W_t, \vec a_0),  y^\nu \right)
 \end{aligned}
\label{eq:algorithm}     
 \end{equation}
\end{myalgorithm}
This class of algorithms defined by the optimizer step $\mathcal{A}(\cdot\,, \gamma, \rho, n_b)$, in which we consider the final gradient update after taking a linear combination of the current iterate and its current gradient (noted as $\Tilde{W}_t(\rho)$ in the above), is broadly used in different contexts. Indeed, routines with positive $\rho$ parameters correspond to Extragradient methods (EgD) \citep{Korpelevich1976TheEM}, while with negative $\rho$ have been recently used in the context of Sharpness Aware Minimization (SAM) \citep{foret2021sharpnessaware}. For clarity, we present our theoretical results with the above, easily interpretable, Algorithm~\ref{algo:optimizer_main_step}. However, our theoretical claims are valid in a more general setting; we refer to Appendices~(\ref{sec:app:proofs}, \ref{sec:app:numerics}) for additional results on more general optimizer steps. 

As previously stated, the central object of our analysis is the efficiency of the network to adapt to the low-dimensional structure identified by $W^\star$. Therefore, we focus on the learned representations by the first layer weights $W_t$ while keeping the second layer weights $\vec a_t = \vec a_0$ fixed during training. This assumption is favorable to performing the theoretical analysis and is largely used in the theoretical community (e.g., \cite{damian2022neural, ba2022high, Bietti2022}).

Numerous theoretical efforts have been devoted to understanding the sample complexity needed to learn multi-index targets. However, an important aspect of many algorithmic routines considered is to exploit a clever \textit{warm start} to initialize the iterates \citep{mondelli2018fundamental, luo2019optimal, maillard2020phase, chen2020learning}. We show in this work that such preprocessing of the data is not needed to learn all polynomial multi-index functions in $T=O(d \log d)$ with the routine in Alg.~\ref{algo:optimizer_main_step}. Instead, we consider a fully agnostic initialization of the network:

\begin{assumption}\label{assump:init}
For each $j \in [p]$, the initial first-layer weight $\vec{w}_{0, j}$ is initialized according to the uniform measure on the sphere:
\[ \vec{w}_{0, j} \overset{i.i.d}{\sim}\operatorname{Unif}(\mathbb{S}^{d-1}). \]
The second-layer weights are drawn i.i.d in $\{-1, 1\}$:
\[ a_{0, j} \overset{i.i.d}{\sim}\operatorname{Unif}(\{-1, 1\}). \]
\end{assumption}

\paragraph{Weak recovery in high dimensions --}
The essential object in our analysis is the evolution of the correlation between the hidden neurons $W = \{\vec w_j\}_{j \in [p]}$ and the target's weights $W^\star$, as a function of the algorithmic time steps. More precisely, we are interested in the number of algorithmic steps needed to achieve an order one correlation with the target's weights $W^\star$, known as \textit{weak recovery}. 
\begin{definition}[Weak recovery]
\label{def:weak_recovery}
    The \textit{target subspace} $V^\star$ is defined as the span of the rows of the target weights $W^\star$:
    \begin{align}
    \label{eq:main:teacher_subspace_def}
        V^\star = \operatorname{span}(\vec w^\star_1, \dots, \vec w^\star_k)
    \end{align}
    We define the weak recovery stopping time for a parameter $\eta \in (0,1)$ independent from $d$ as:
\begin{equation}
    t^{+}_{\eta} = \operatorname{min}\{t \geq 0: \lVert W W^{\star \top} \lVert_F \geq \eta\}
\end{equation}
\end{definition}

\paragraph{From Information Exponents to Generative Exponents --}Recent works have sharply theoretically characterized the weak recovery stopping time for a variety of multi-index targets learned with one-pass SGD. \cite{BenArous2021} for the single-index case, then \cite{abbe2022merged} for the multi-index one have unveiled the presence of an Information Exponent $\ell$ characterizing the time complexity needed to weakly recover the target subspace $V^\star$, defined in the single-index case as
\begin{equation}
    \ell = \min \{ k \in \N \::\: \Eb{x \sim \cN(0, 1)}{h^\star(x)H_k(x)} \neq 0 \}.
\end{equation}
where $H_k$ is the $k$-th Hermite polynomial \citep{odonnell_2014_analysis}.
Under this framework, the sample complexity required to achieve weak recovery when the link function has Information Exponent $\ell$ is \citep{BenArous2021}
\begin{align} \label{eq:gba_time_scaling}
    T(\ell) = 
        \begin{cases}
        &O(d^{\ell -1} ) \qquad \hspace{0.8em}\text{if $\ell>2$} \\
        &O(d \log{d} ) \qquad \text{if $\ell=2$} \\
        &O(d) \qquad \hspace{2.2em}\text{if $\ell=1$}.
    \end{cases}
\end{align}
These bounds have been improved by \cite{damian_2023_smoothing} up to order $d^{\ell/2}$; the latter matches the so-called \emph{Correlated Statistical Query} lower bound, which considers queries of the form $\Ea{y \phi(\vec{z})}$.

However, this Information Exponent is not the right measure of complexity for low-dimensional weak recovery. Some iterative algorithms, for instance, are known to perform drastically better\citep{barbier2019optimal,celentano2021highdimensional,mondelli2018fundamental,luo2019optimal,maillard2020phase,troiani2024fundamental}. Recently \cite{damian2024computational} introduced a new Generative Exponent $\ell^\star$ (generative exponent) governing the weak recovery time scaling for algorithms in the Statistical Query (SQ) or Low Degree Polynomial (LDP) family. This exponent is defined as
\begin{equation}
    \ell^\star = \min \{ k \in \N \::\: \Eb{x \sim \cN(0, 1)}{f(h^\star(x))H_k(x)} \neq 0 \ \text{for some} \ f: \R \to \R \},
\end{equation}
which allows for the application of arbitrary transformations of the labels before performing the queries. Under this definition, the ``hard'' problems with $\ell^\star > 2$ are also impossible to learn for the best first-order algorithm, \emph{Approximate Message Passing} \citep{troiani2024fundamental}. 

We discuss in this work the extension of the generative exponent to the multi-index setting and investigate the dynamics of specific gradient-based programs (Algorithm~\ref{algo:optimizer_main_step}) as a function of the freshly defined hardness exponent $\ell^*$. 
These theoretical findings match the recent one of \cite{kou2024matching} on the computational efficiency of SGD in learning sparse $k-$parity (associated with $\ell^\star = k$).

When the optimal transformation $f$ is known, a simple SGD algorithm on $(\vec{z}, f(y))$ achieves the Generative Exponent lower bound. However, one problem remains: does there exist a class of SGD-like algorithms that can achieve this lower bound agnostically?

\paragraph{Achieving lower bounds with data repetition --}

Although the performance of single-pass SGD is limited by the Information Exponent of $h^\star$, the situation drastically changes when multiple-pass algorithms are considered. Recently, \cite{dandi2024benefits} proved that non-even single-index targets are weakly recovered in $T = O(d)$ when considering extensive batch sizes with multiple pass. This begs the question of how good can a ``reusing'' algorithm be. We answer this question for the class of polynomial link functions:
\begin{theorem}[Informal]
    There is a choice of hyperparameters such that if $h^\star$ is a polynomial function, Algorithm~\ref{algo:optimizer_main_step} achieves weak recovery in $O(d\log(d)^2)$ samples.
\end{theorem}
This shows in particular that a simple gradient-based algorithm, which only requires seeing the data twice, is optimal for learning any polynomials.

\paragraph{Learning representations --}
Learning relevant features from data is a fundamental property of neural networks. However, much theoretical work has focused on lazy training in two-layer networks \citep{jacot2018neural}, where features are not effectively learned. In this regime, shallow networks behave as kernel machines, which struggle to adapt to low-dimensional structures in data \citep{Ghorbani2020}. Understanding representation learning and corrections to the kernel regime is thus a key challenge in machine learning theory \citep{pmlr-v75-dudeja18a, atanasov2022neural, Bietti2022, damian2022neural, petrini2022learning}. In this manuscript we provably characterize the number of iterations needed from Algorithm~\ref{algo:optimizer_main_step} to learn the relevant features of the data, i.e., attain weak recovery of the target subspace. Our results provide the scaling of the relevant hyperparameters $(\gamma, \rho)$ with the diverging input dimension $d$. The minibatch size will be fixed to $n_b =1$ in the main body, and we refer to Appendix~\ref{sec:app:numerics} for the extension to larger batch sizes.

\subsection*{Summary of Main Results}
\begin{itemize}[noitemsep,leftmargin=1em,wide=0pt]
    \item We prove for single-index targets that one-pass gradient-based algorithms surpass the Correlation Statistical Query (CSQ) limits established by the Information Exponent \citep{BenArous2021}.  
    \item We unveil that the dynamics of the family of Alg.~\ref{algo:optimizer_main_step} are governed by the generative exponent \cite{damian2024computational}, with an additional polynomial restriction on the transformation of the output.
    \item We generalize the notion of generative exponent to multi-index targets. As in the single-index case, one-pass gradient-based algorithms can overcome CSQ performance dictated by the Leap Coefficient \citep{abbe2022merged}. Our definition is consistent with the one of \citep{troiani2024fundamental}. 
    \item We prove that all polynomial multi-index functions are learned by Algorithm~\ref{algo:optimizer_main_step} either with total sample complexity $O(d)$ or $O(d \log(d)^2)$ if associated with the presence of a symmetry. 
    \item  The implementation of Algorithm~\ref{algo:optimizer_main_step} does not require pre-processing to initially correlate the estimation with the ground truth, and learns the meaningful representations from data alone. This is in contrast numerous findings (e.g., \cite{mondelli2018fundamental, luo2019optimal, maillard2020phase} for single index targets or \cite{chen2020learning} for multi-index ones).
    \item We characterize the class of hard functions not learned  in (almost) linear sample complexity. We show, however, that such functions can be learned through an hierarchical mechanism that extends the CSQ staircase first developed in \cite{abbe2021staircase} to a different, larger set, of functions.
    \item We validate the formal theoretical claims with detailed numerical illustrations. The code to reproduce representative figures is available at \url{https://github.com/IdePHICS/Repetita-Iuvant}.
\end{itemize}
\section{Single-Index Model}
\subsection{Main results}
We first consider, for simplicity, the class of single-index models, in which $k \!=\! p \!=\! 1$. This corresponds to a mismatched setting of Generalized Linear Models, in which we want to learn
\[ y = h^\star(\langle \vec{w^\star}, \vec{z} \rangle), \quad \text{with} \quad f(\vec{z}, \vec{w}, a) = a\sigma(\langle \vec{w}, \vec{z} \rangle). \]
In this section we rigorously characterize the number of iterations needed for the class of Algorithms~\ref{algo:optimizer_main_step} to perform weak recovery (as in Definition~\ref{def:weak_recovery}) in the context of single-index targets. We establish a clear separation between the the learning efficiency of the algorithmic  family~\ref{algo:optimizer_main_step} and One-Pass SGD, limited by the Information Exponent of the target to be learned \citep{BenArous2021}.

We start by introducing a restriction of the generative information exponent in \cite{damian2024computational} to polynomial transformations.
\begin{definition}[Polynomial Generative Information exponent]
\label{def:lstar_damian2024}
We define $\ell^\star_p$ as the smallest integer $k$ such that there exists a polynomial $p: \R \rightarrow \R$ with:
\begin{equation}\label{eq:f_hard_gen}
        \Eb{x\sim\cN(0, 1)}{p(h^\star(x))H_{k}(x)}\neq 0,
    \end{equation}
\end{definition}

Our assumptions on the SGD implementation and the target function are as follows:
\begin{assumption}\label{assump:spherical_corr}
    Algorithm \ref{algo:optimizer_main_step} is run with single-sample steps ($n_b = 1$), and uses the correlation loss
    \[ \cL(y, \hat y) = 1 - y\hat y.\]
    We also consider a spherical version of Alg. \ref{algo:optimizer_main_step}, with
    \begin{equation}\label{eq:spherical_sam}
        \vec{w}_{t+1} = \frac{\vec w_{t} - \gamma\nabla^\bot_{\vec w}
     \cL
    \left(
    f (\vec z^\nu; 
    \tilde{\vec{w}}_t(\rho), a_0),  y^\nu  \right)}{\norm{\vec w_{t} - \gamma\nabla^\bot_{\vec w}
     \cL
    \left(
    f (\vec z^\nu; 
    \tilde{\vec{w}}_t(\rho), a_0),  y^\nu  \right)}}
    \end{equation}
    where $\nabla^\bot f = \nabla f - \langle \nabla f, \vec{w} \rangle \vec{w}$ is the spherical gradient.
\end{assumption}

\begin{assumption}\label{assump:derivatives}
    The activation $\sigma$ is analytic almost everywhere, and $\sigma^{(n)}$ satisfies the polynomial growth assumption \eqref{eq:poly_growth} for $n \in \mathbb{N}$ (possibly for different constants $C_n$). Further, for all $n \geq 1$, we have
    \[ \Eb{x\sim\cN(0, 1)}{\sigma^{(n)}(x) \sigma'(x)^{n-1}} \neq 0 \quad \text{and} \quad \Eb{x\sim\cN(0, 1)}{x \sigma^{(n)}(x) \sigma'(x)^{n-1}} \neq 0 \]
\end{assumption}

We are now ready to state our main result.
\begin{theorem}
\label{thm:single_index_recovery}
    Suppose that Assumptions \ref{assump:poly_growth}-\ref{assump:derivatives} hold, and let $\rho = \rho_0 d^{-1}$. Then, under an event of probability $1/2$ on the initialization $(a_0, \vec{w}_0)$, for each $\delta > 0$, there exist constants $\gamma_0(\delta)$ and $C(\delta)$ such that the following holds:
    \begin{itemize}
        \item If $\ell_p^\star = 1$, choosing $\gamma = \gamma_0(\delta)d^{-1}$, then
$
             \dP(t^+_\eta \leq C(\delta) \cdot d) \geq 1 - \delta
     $,

        \item If $\ell_p^\star = 2$, choosing $\gamma = \gamma_0(\delta)[d\log(d)]^{-1}$, then 
        $
            \dP(t^+_\eta \leq C(\delta) \cdot d \log(d)^2) \geq 1 - \delta
       $.
    \end{itemize}
    The above holds for almost every choice of $\rho_0$ under the Lebesgue measure. In particular, in a network with $p = O(\log\frac1\delta)$ neurons, at least one of them achieves weak recovery with probability $1 - \delta$.
\end{theorem}

\paragraph {Proof Sketch --} We provide here the proof sketch for Theorem~\ref{thm:single_index_recovery}. For any vector $\vec{w}$, under the correlation loss, we have
\begin{equation}\label{eq:sgd_gradient}
    \nabla_{\vec{w}} \cL(f(\vec{z}; \vec{w}, a_0), y) = a_0 y \sigma'(\langle \vec{w}, z \rangle) \cdot \vec{z},
\end{equation}
so the gradient is aligned with $\vec{z}$. As a result, we have
\begin{equation}\label{eq:sam_gradient}
    \nabla_{\vec w}\cL\left(f(\vec z; \tilde{\vec{w}}(\rho), a_0),  y \right) = a_0 y \sigma'\big(\langle \vec{w}, \vec{z} \rangle + a_0 \rho y \sigma'(\langle \vec{w}, \vec{z} \rangle) \cdot \underbrace{\norm{\vec{z}}^2}_{\approx d}\big) \cdot \vec{z}
\end{equation}
This expression for the gradient exhibits two important properties:
\begin{itemize}
    \item Even though $\rho \asymp d^{-1}$, since the gradient lies along $\vec{z}$, the additional dot product with $\vec{z}$ amplifies the signal by a factor of $d$,
    \item The resulting gradient  \eqref{eq:sam_gradient} is a \emph{non-linear} function of $y$, as opposed to the linear one of \eqref{eq:sgd_gradient}. 
\end{itemize}
The latter property enables us to show that the dynamics driven by equation \eqref{eq:sam_gradient} can implement all polynomial transformations of the output $y$, for almost all choices of $\rho_0$. Subsequently, the SGD algorithm on the transformed input $p(y)$ can be studied with the same techniques as \cite{BenArous2021}, but with an information exponent equal to $\ell_p^\star$.

\subsection{Discussion}\label{subsec:single_index_discussion}

\paragraph{Significance and necessity of the assumptions --}
Assumption \ref{assump:spherical_corr} simplifies Alg.~\ref{algo:optimizer_main_step} to a more tractable version. In particular, the normalization step allows us to keep track of only one parameter, the overlap between $\vec{w}_t$ and $\vec{w}^\star$. We expect however Thm.~\ref{thm:single_index_recovery} to hold in the general setting of Alg.~\ref{algo:optimizer_main_step}.

Assumption \ref{assump:derivatives} is a regularity assumption, ensuring that some quantities of interest in the proof are non-zero. The first part is satisfied by virtually every activation function; the second is more restrictive, but we show that it is satisfied for biased activations:
\begin{lemma}\label{lem:bias_activation}
    Let $\sigma$ be a non-polynomial analytic function. Then the second part of Assumption \ref{assump:derivatives} is true for the function $x \mapsto \sigma(x + b)$ for almost every choice of $b$ (according to the Lebesgue measure).
\end{lemma}

Finally, although we state our result as an almost sure event over $\rho_0$, it can be reformulated into an event on the second layer weight $a_0$, in which the randomness comes from the initialization step.

\paragraph {Learning polynomial functions --} Although our definition of the Polynomial Generative Information exponent is more restrictive than the one in \cite{damian2024computational}, we show that it is enough to allow learning of a large class of functions, namely polynomials:
\begin{theorem}\label{thm:poly_pge}
    Assume that $h^\star$ is a polynomial function. Then the Polynomial Generative Information exponent of $h^\star$ is always at most 2, and is equal to 1 whenever $h^\star$ is not even.
\end{theorem}
This theorem implies that the Extragradient algorithm class allows us to match the sample complexity of the algorithm of \cite{chen2020learning}, without the need for \emph{ad hoc} preprocessing.


\paragraph{Beyond Algorithm~\ref{algo:optimizer_main_step} --} A succinct summary of the proof goes as follows: the first time we see a sample $\vec{z}$, we store information about $y$ along the direction of $\vec{z}$. The second time we see the same sample, the component along $\vec{z}$ in $\vec{w}$ interact \emph{non-linearly} with $y$, which bypasses the CSQ framework. Importantly, we expect that this behavior also appears in SGD algorithms where we do not see the data twice \emph{in a row}, but \emph{across multiple epochs}! Indeed, upon small enough correlation between the $\vec{z}^\nu$, seeing multiple examples before the repetition of $\vec{z}$ should not interfere with the data stored along this direction.
We are therefore not claiming that Algorithm~\ref{algo:optimizer_main_step} is superior to other classically-used methods; rather, it provides a tractable and realistic setting to study data repetition in SGD, which was ignored in previous theoretical works on the topic. 

\begin{figure}[t]
    \centering
    \subfigure[$\scriptstyle h^\star(x) = \operatorname{relu}(x)$]{%
        \includegraphics[height=0.21\textwidth]{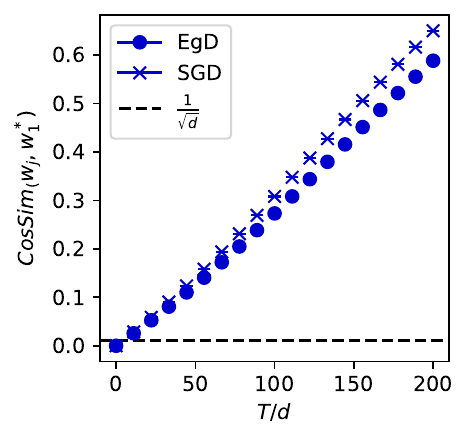}
        \label{fig:single_index_1}
    }\hfill
    \subfigure[$\scriptstyle h^\star(x) = \mathrm{He}_2(x)$]{%
        \includegraphics[height=0.21\textwidth]{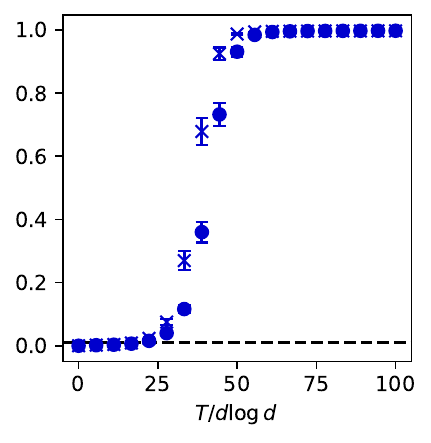}
        \label{fig:single_index_2}
    }\hfill
    \subfigure[$\scriptstyle h^\star(x) = \mathrm{He}_3(x)$]{%
        \includegraphics[height=0.21\textwidth]{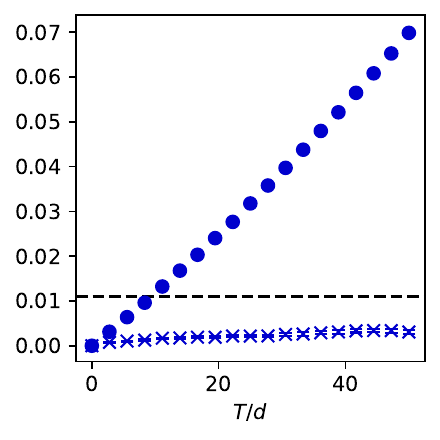}
        \label{fig:single_index_3}
    }\hfill
    \subfigure[$\scriptstyle h^\star(x) = \mathrm{He}_4(x)$]{%
        \includegraphics[height=0.21\textwidth]{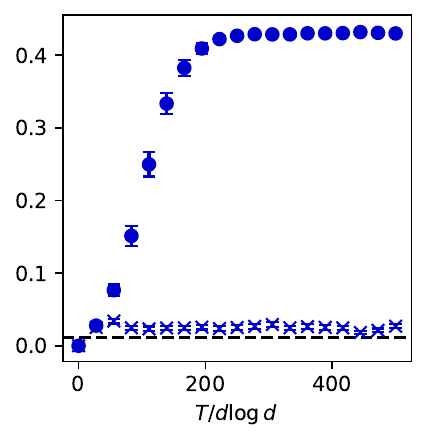}
        \label{fig:single_index_4}
    }

    \caption{\textbf{Learning single-index targets --}
    Evolution of the Cosine Similarity attained by EgD (crosses) and SGD (dots) as a function of the normalized iteration time.
    (dashed horizontal line $\frac{1}{\sqrt{d}}$ is a visual guide to place random performance).
    \textbf{(a) $(\ell,\ell^\star)=(1,1)$:} both algorithms learn in linear time.
    \textbf{(b) $(\ell,\ell^\star)=(2,2)$:} both algorithms learn in $O(d\log d)$ time.
    \textbf{(c) $(\ell,\ell^\star)=(3,1)$:} EgD learns in linear time, while SGD requires $O(d^2)$ time.
    \textbf{(d) $(\ell,\ell^\star)=(4,2)$:} EgD learns in $O(d\log d)$ time,  SGD requires $O(d^3)$ time
    (Details in App.~\ref{sec:app:implementation}).}
    \label{fig:main:single_index_complete}
\end{figure}

\subsection{Numerical illustrations}\label{subsec:single_numerics}
We illustrate in Fig.~\ref{fig:main:single_index_complete} the stark difference between the weak recovery dynamics of One-Pass SGD and Algo.~\ref{algo:optimizer_main_step} for single-index targets. The Cosine Similarity between the learned weight $\vec w_t$ and the ground truth informative direction $\vec w_\star$ is shown as a function of the time steps $t$ for different target functions. Two optimization routines are considered to exemplify the learning behaviour: a) Extragradient Descent (EgD), corresponding to the family of Algorithms~\ref{algo:optimizer_main_step} with positive $\rho$ parameters; b) One-Pass SGD, corresponding to vanilla SGD, or equivalently Algorithm~\ref{algo:optimizer_main_step} associated with $\rho=0$ hyperparameter. The scaling of $(\gamma, \rho)$ as a function of the input dimension  $d$ and exponent $\ell^\star$ are given in Thm.~\ref{thm:single_index_recovery}, while the mini-batch size is fixed to $n_b =1$ (See App.~\ref{sec:app:numerics} for extension to larger $n_b$). For all plots, we take $\sigma=\mathrm{relu}$ and we refer to App.~\ref{sec:app:implementation} for details.
\begin{itemize}[noitemsep,leftmargin=1em,wide=0pt]
    \item \textbf{SGD-easy non-symmetric targets ($\ell = \ell^\star =1$) --} The Left section of Fig.~\ref{fig:main:single_index_complete} shows $h^\star(x) = \mathrm{relu}(x)$. Here both SGD and Algorithm~\ref{algo:optimizer_main_step} learn in $T= O(d)$.
    \item \textbf{SGD-easy symmetric targets ($\ell = \ell^\star = 2$) -- } The Center-Left section of Fig.~\ref{fig:main:single_index_complete} shows $ h^\star(x) = \mathrm{He}_2(x)$. Here both SGD and Algorithm~\ref{algo:optimizer_main_step} learn in $T=O(d \log d)$. 
    \item \textbf{SGD-hard non-symmetric targets ($\ell > 2, \ell^\star =1$) -- }  The Center-Right section of Fig.~\ref{fig:main:single_index_complete} shows $ h^\star(x) = \mathrm{He}_3(x)$. Here Algortihm~\ref{algo:optimizer_main_step} learns in $T=O(d)$ while One Pass SGD suffers from the limitations detailed in eq.~\eqref{eq:gba_time_scaling}, i.e. $T= O(d^{\ell-1})$ with $\ell=3$ for this case.
    \item \textbf{SGD-hard symmetric targets ($\ell > 2, \ell^\star =2$) -- } The Right section of Fig.~\ref{fig:main:single_index_complete} shows $h^\star(x) = \mathrm{He}_4(x)$. Here Algortihm~\ref{algo:optimizer_main_step} learns in $T=O(d \log d)$ while One Pass SGD suffers from the time scaling law detailed in eq.~\eqref{eq:gba_time_scaling} $T= O(d^{\ell-1})$.
\end{itemize}

\begin{figure}[t]
    \centering
    \subfigure[{$\scriptstyle h^\star(\vec x) = x_1+x_1x_2$}]{%
        \includegraphics[height=0.22\textwidth, trim={0 8 0 0}, clip]{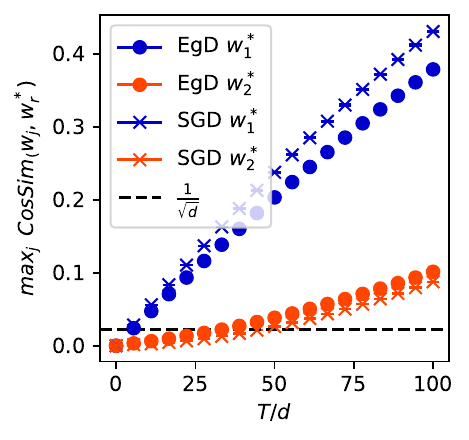}
        \label{fig:multi_index_1}
    }\hfill
    \subfigure[{$\scriptstyle h^\star(\vec x) = \mathrm{sign}{(x_1x_2)}$}]{%
        \includegraphics[height=0.22\textwidth, trim={0 8 0 0}, clip]{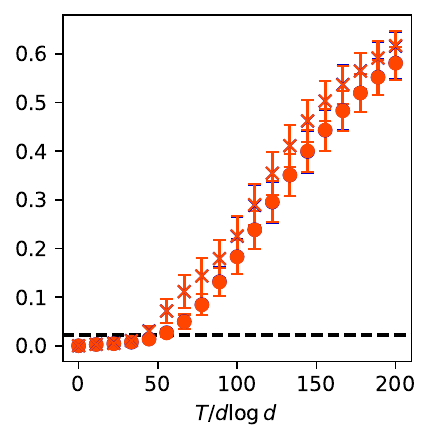}
        \label{fig:multi_index_2}
    }\hfill
    \subfigure[{$\scriptstyle h^\star(\vec x) = x_1+\mathrm{He}_3(x_2)$}]{%
        \includegraphics[height=0.22\textwidth, trim={0 8 0 0}, clip]{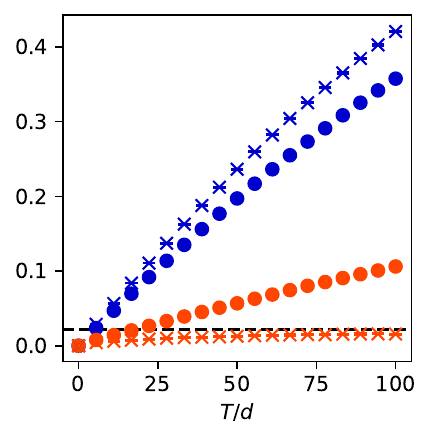}
        \label{fig:multi_index_3}
    }\hfill
    \subfigure[{$\scriptstyle h^\star (\vec x) = x_1x_2x_3$}]{%
        \includegraphics[height=0.22\textwidth, trim={0 8 0 0}, clip]{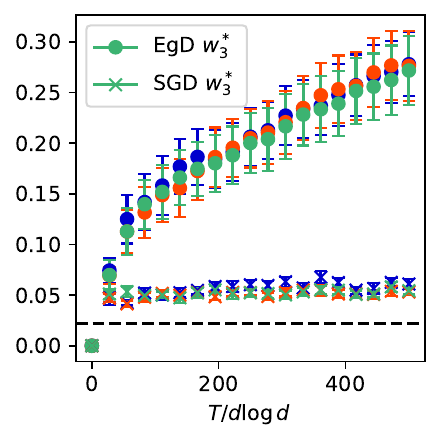}
        \label{fig:multi_index_4}
    }

    \caption{\textbf{Learning multi-index targets --}
    Evolution of the maximum Cosine Similarities attained by EgD (crosses) and SGD (dots) as a function of the normalized iteration time. The Different directions $\{\vec w^\star_r\}_{r \in [k]}$ are identified by colors: $\vec w^\star_1$ (blue), $\vec w^\star_2$ (orange), $\vec w^\star_3$ (green).
     (dashed horizontal line $\frac{1}{\sqrt{d}}$ is a visual guide to place random performance). 
    \textbf{(a)~$(\ell,\ell^\star)=(1,1)$:} both  algorithms learn the first direction in linear time, as well the second one in $O(d)$ steps using the staircase mechanism.
    \textbf{(b) $(\ell,\ell^\star)=(2,2)$:} both  algorithms learn the two directions simultaneously in $O(d\log d)$ steps.
    \textbf{(c)~$(\ell_{\vec w^\star_2},\ell_{\vec w^\star_2}^\star)=(3,1)$:} both algorithms learn $\vec w^\star_1$ in linear time, but only EgD learns $\vec w^\star_2$ in $O(d)$; SGD instead requires $O(d^2)$.
    \textbf{(d) $(\ell,\ell^\star)=(3,2)$:} EgD learns all 3 directions simultaneously in $O(d\log d)$ steps, while SGD needs $O(d^2)$ time.  (Details in App.~\ref{sec:app:implementation}).}
    \label{fig:main:multi_index_complete}
\end{figure}

\section{Multi-Index Model}
We now investigate the superiority of Algorithm~\ref{algo:optimizer_main_step} over the idealized One-Pass SGD training scheme when learning multi-index targets. In this setting, the networks weights can align with only a few select directions from $V^\star$. To quantify this behavior, we define directional versions of the Information and Generative Exponents (see also \citep{troiani2024fundamental}):
\begin{definition}\label{def:ie_pge_multi}
    Let $\vec{v}\in V^\star$. The Information Exponent $\ell_{\vec{v}}$ (resp. Polynomial Generative Exponent $\ell^\star_{\vec{v}}$) of $y$ in the direction $\vec v$ is the smallest $k$ such that
    \[ \Ea{yH_k(\langle \vec{v}, \vec{z} \rangle)} \neq 0 \quad (\text{resp.} \  \Ea{p_{\vec{v}}(y)H_k(\langle \vec{v}, \vec{z} \rangle)} \neq 0 ) \]
    for a polynomial $p_{\vec{v}}$. We also let $\ell = \min_{\vec{v}}\ell_{\vec{v}}$ and $\ell_p^\star = \min_{\vec{v}}\ell^\star_{\vec{v}}$.  The minimal subspace $U^\star \subseteq V^\star$ is defined as the smallest subspace $U$ such that $\ell_\vec{v} > \ell$ (resp. $\ell_{\vec{v}}^\star > \ell_p^\star$) for every $v\in U^\bot$.
\end{definition}
The identification of the class of hard functions has been subject of intense theoretical scrutiny \citep{abbe2021staircase,abbe2022merged,abbe2023sgd,dandi2023twolayer,bietti2023learning,troiani2024fundamental}; and for the problem of initial alignment it was shown to depend on the Information exponent defined above. 
Thus, similarly to the single-index scenario, the time complexity needed for Algorithm~\ref{algo:optimizer_main_step} to weakly recover the target subspace follows the law in eq.~\eqref{eq:gba_time_scaling} for the multi-index information exponent $\ell$. A natural question is thus to ask whether the class of Algorithm~\ref{algo:optimizer_main_step} also allows us to bypass this requirement. We show that this is indeed the case. Consider a neural network as in Equation~\eqref{eq:def_nn}; we assume that the internal learning rate $\rho$ is chosen independently for each neuron:
\begin{assumption}
    For each $j \in [p]$, the parameter $\rho_{0, j}$ is chosen independently from all other neurons, according to a \emph{symmetric} distribution which is absolutely continuous w.r.t the Lebesgue measure.
\end{assumption}

\begin{theorem}\label{thm:multi_index_recovery}
    Let $\delta > 0$, and $\Pi^\star$ the projection operator on $V^\star$. There exists constants $\gamma_0(\delta), C(\delta)$ depending only on $\delta$ such that if we take $p \geq C(\delta)k$ and
    \begin{itemize}
        \item if $\ell_p^\star = 1$, $\gamma = \gamma_0(\delta) d^{-1}$ and $T = C(\delta)d$,
        \item if $\ell_p^\star = 2$, $\gamma = \gamma_0(\delta) (d\log(d))^{-1}$ and $T = C(\delta)d\log(d)^2$,
    \end{itemize}
    then with probability $1 - \delta$, there exists a $t \leq T$ such that at least a positive proportion of the vectors $\vec{w}_{t, i}$ satisfy $\norm{\Pi^\star \vec{w}_{t, i}} \geq \eta$.
Further, if $\ell_p^\star = 1$, then the vectors $\{\Pi^\star \vec{w}_{t, i}\}_{i \in [p]}$ span the minimal subspace $U^\star$.
\end{theorem}

\begin{figure}
    \centering
    \subfigure[ ${\scriptstyle h^\star(\vec x)= x_1 + x_1 \mathrm{He}_3(x_2)}$]{%
        \includegraphics[height=0.23\textwidth, trim={0 8 0 0}, clip]{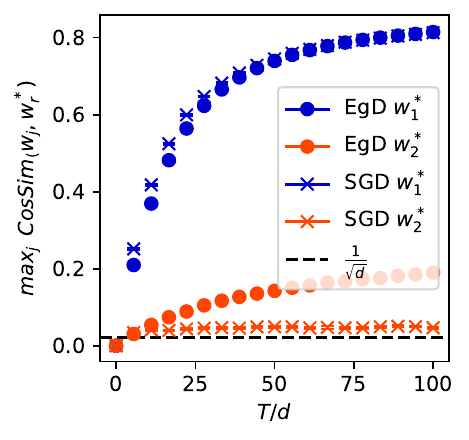}
        \label{fig:staircase_1}
    }
    \hfill \subfigure[${\scriptstyle h_{\mathrm{sign}}^\star(\vec{x}) = \mathrm{sign}(x_1x_2x_3), \quad h^\star_{\mathrm{stair}}(\vec x) = \mathrm{He}_2(x_1) + \mathrm{sign}(x_1x_2x_3)}$]{%
        \includegraphics[height=0.23\textwidth, trim={0 8 0 0}, clip]{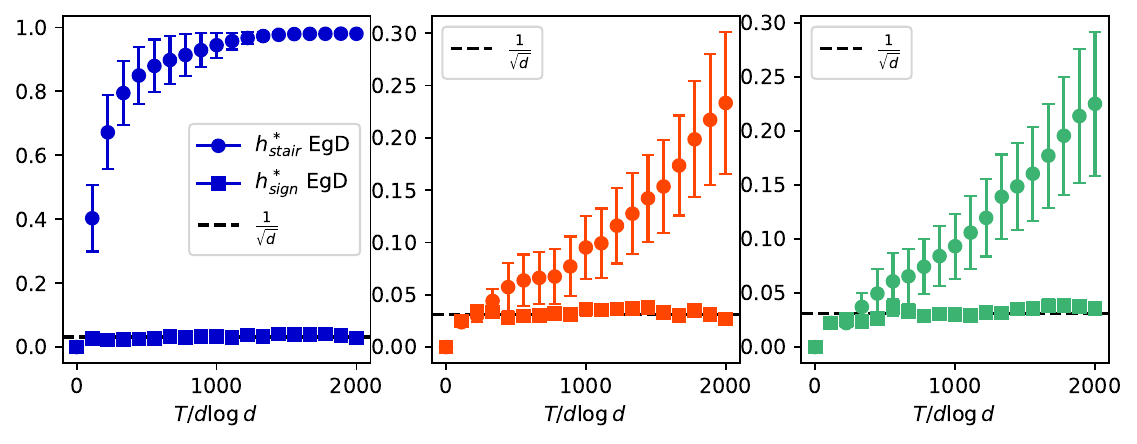}
        \label{fig:staircase_2}
    }
    \caption{\textbf{Hierarchical learning --} Evolution of the maximum Cosine Similarities with different target directions $\{\vec w^\star_r\}_{r \in [k]}$ are identified by different colors: $\vec w^\star_1$ (blue), $\vec w^\star_2$ (orange), $\vec w^\star_3$ (green) as a function of the normalized iteration time.
    \textbf{(a):} both the algorithms learn $\vec w^\star_1$ in linear time, but only EgD can also learn $\vec w^\star_2$ in $O(d)$ via a staircase mechanism; SGD requires another $O(d^2)$ steps to take advantage of the staircase and learn $\vec w^\star_2$.
    \textbf{(b):} EgD performance for the target function $f^\star_{\mathrm{stair}}$ (dots) and $f^\star_{\mathrm{sign}}$(squares); we plot in separate subplots the overlap with different directions to highlight the presence of hierarchical learning mechanism. Learning the first target direction $\vec w^\star_1$ triggers the hierarchical mechanism and EgD is able to weakly recover the full target subspace $\mathrm{Span}(\vec w^\star_1, \vec w^\star_2, \vec w^\star_3)$ in $O(d \log d)$, while this does not happen when removing $\mathrm{He}_2(x_1)$ from the target.
    See App.~\ref{sec:app:implementation} for additional details.}
    \label{fig:main:staircase1}
\end{figure}

\subsection{Numerical investigation with multi-index target}\label{subsec:multi_numerics}

We illustrate the results of Thm.~\ref{thm:multi_index_recovery} numerically by comparing the performance of One-Pass SGD with Alg.~\ref{algo:optimizer_main_step} in Fig~\ref{fig:main:multi_index_complete}. We measure the maximal cosine similarity between the hidden layer neurons and the different rows of the target's weight matrix $W^\star$ as a function of the iteration time. We set without loss of generality the target's directions to the standard basis of $\mathbb{R}^d$ to ease the notation. The algorithms considered are again EgD and vanilla One-Pass SGD (See App.~\ref{sec:app:implementation} for details).
\begin{itemize}[noitemsep,leftmargin=1em,wide=0pt]
\item \textbf{SGD-easy non-symmetric targets ($\ell \!=\! \ell^\star \!=\! 1$) --} Fig.~\ref{fig:main:multi_index_complete}, left, shows $h^\star(x_1,x_2) = x_1 + x_1x_2$. Here both SGD and Alg.~\ref{algo:optimizer_main_step} learn in $T= O(d)$ the directions $(\vec w^\star_1, \vec w^\star_2)$.

\item \textbf{SGD-easy symmetric targets ($\ell \!=\! \ell^\star \!=\! 2$) --} Fig.~\ref{fig:main:multi_index_complete}, center left, shows $h^\star(x_1,x_2) = x_1x_2$.  Here both SGD and Alg.~\ref{algo:optimizer_main_step} learn in $T=O(d \log d)$ the directions $(\vec w^\star_1, \vec w^\star_2)$. 

\item \textbf{Non-symmetric targets with SGD-easy and hard directions ($\ell_{\vec{v}} > 2, \ell^\star_{\vec{v}}=1$) --} Fig.~\ref{fig:main:multi_index_complete}, center-right, shows $h^\star(x_1,x_2) \!=\! x_1 \!+\! \mathrm{He}_3(x_2) $. Here both SGD and Alg.~\ref{algo:optimizer_main_step} learn in $T=O(d)$ the direction $\vec w^\star_1$, that satisfies $\ell_{\vec{w_1^\star}} = \ell_{\vec{w_1}^\star}^\star = 1$. However, since $\ell_{\vec{w_2^\star}} = 3$ while $\ell_{\vec{w_2^\star}}^\star = 1$, One-Pass SGD suffers from the limitations detailed in eq.~\eqref{eq:gba_time_scaling} and requires $\Omega(d^2)$ samples to learn the second direction $\vec w^\star_2$. This contrasts with the behaviour of Alg.~\ref{algo:optimizer_main_step} that learns also $\vec w^\star_2$ in $T=O(d)$ steps.

\item \textbf{SGD-hard symmetric targets ($\ell > 2,  \ell^\star =2$) --} The right section Fig.~\ref{fig:main:multi_index_complete} shows $h^\star(x_1,x_2,x_3) = x_1 x_2 x_3$. The Information Exponent is $\ell=3$, hence vanilla SGD does not learn any direction in the target subspace in $T=O(d \log d)$ iterations. However, Algorithms~\ref{algo:optimizer_main_step} learns in $T=O(d \log d)$ steps all the three directions $\{\vec w^\star_1, \vec w^\star_2, \vec w^\star_3\}$ since the Generative Exponents $\ell_{\vec{w}^\star_r}^\star$ are all equal to two.
\end{itemize}

\subsection{Learning hard functions through a hierarchical mechanism}\label{subsec:staircase}

The investigation of the hierarchical nature of SGD learning has attracted noticeable attention \citep{abbe2021staircase,abbe2022merged,abbe2023sgd, dandi2023twolayer}. In this paper, we portray a completely different picture in terms of computational efficiencies when data repetition is considered in the algorithmic SGD routine. One may wonder if there is a generalization of such a hierarchical learning mechanism to the present novel setting. Indeed for optimal AMP algorithm, a diffrent picture emerges, (called \emph{Grand staircase} in \cite{troiani2024fundamental}). Here we show that coupling between directions can hierarchically guide the learning process as well with SGD.

\begin{figure}[t]
    \centering
    \vspace{-.4cm}
    \includegraphics[width=0.75\linewidth]{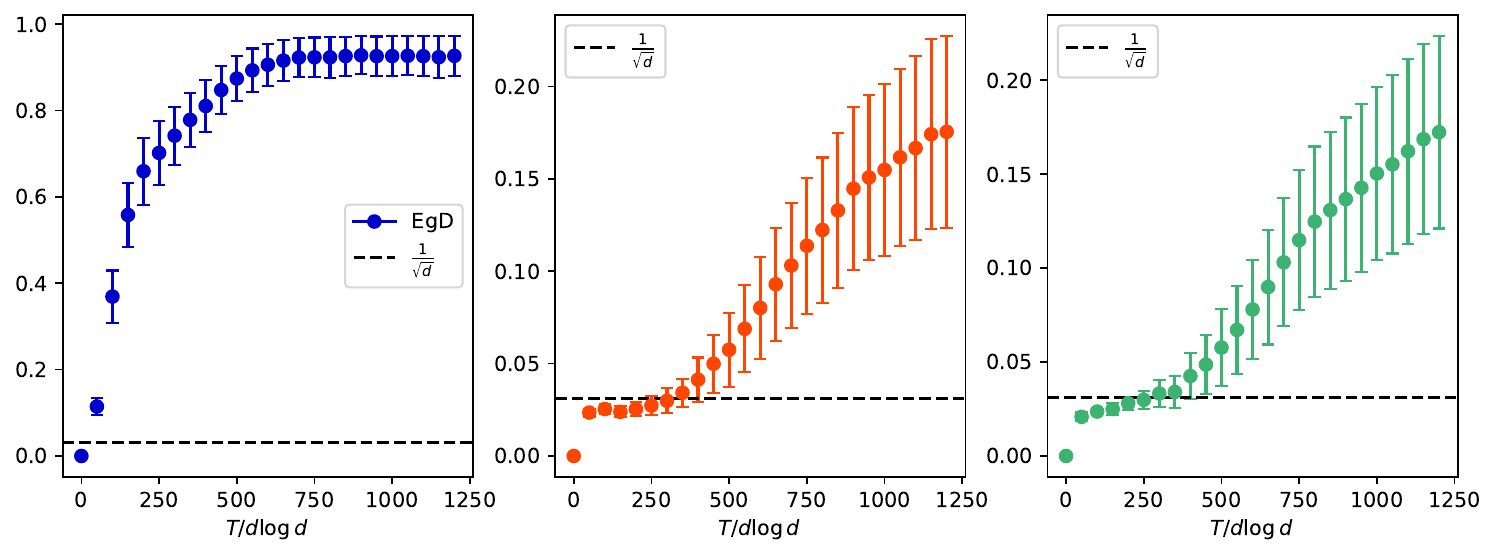}
    \caption{\textbf{Hierarchical learning --} Evolution of the maximum Cosine Similarities attained by EgD with different target directions $\{\vec w^\star_r\}_{r \in [k]}$ are identified by different colors: $\vec w^\star_1$ (blue), $\vec w^\star_2$ (orange), $\vec w^\star_3$ (green) as a function of the normalized iteration time. The target function is $f^\star_{\mathrm{sq-stair}}(\vec z)= \mathrm{He}_4(z_1) + \mathrm{sign}(z_1z_2z_3)$, that is an SQ-staircase (or grand staircase), not a CSQ-staircase. EgD learn $\vec w^\star_1$ in \(O(d\log d)\) steps and use the information to learn the other two directions in another \(O(d\log d)\) steps. SGD (not showed in the plot) cannot take advantage of the staircase mechanism since $\mathrm{He}_4(z_1)$ and $\mathrm{sign}(z_1z_2z_3)$ have information exponent \(\ell = 4\) and \(\ell = 3\) respectively. (See App.~\ref{sec:app:implementation}).}   
    \label{fig:main:staircase2}
\end{figure}
The Left panel in Fig.~\ref{fig:main:staircase1} shows an example of the so-called staircase functions \citep{abbe2021staircase}, i.e. $f^\star(\vec z) = z_1 + z_1 \mathrm{He}_3(z_2)$. While one-pass SGD needs to learn hierarchically first the direction $\vec w^\star_1$ in $T=O(d)$ and then $\vec w^\star_2$ in $T=O(d^2)$, Algorithm~\ref{algo:optimizer_main_step} easily learns both in linear time. This observation could lead to infer that hierarchical learning mechanisms are not present when more realistic training scenarios are considered.

To refute this, we illustrate in the right panel of Fig.~\ref{fig:main:staircase1} a scenario that precisely exemplifies the presence of hierarchical mechanisms within Algorithm~\ref{algo:optimizer_main_step}. We run this algorithm on two different target functions, $f_{\mathrm{sign}}^\star(\vec{z}) = \mathrm{sign}(z_1z_2z_3)$ and $f^\star_{\mathrm{stair}}(\vec z) = \mathrm{He}_2(z_1) + \mathrm{sign}(z_1z_2z_3)$. We show in App. \ref{sec:app:proofs} that $3-$sparse parities of the form $\mathrm{sign}(z_1z_2z_3)$ are hard functions even for SGD with data repetition (Algorithm~\ref{algo:optimizer_main_step}), and indeed they are not learned in (almost) linear time. On the other hand, for the function $f^\star_{\mathrm{stair}}$, our simulations predict that the first direction $\vec w^\star_1$ is learned in $T = O(d \log d)$ iterations (see the rightmost panel of Fig.~\ref{fig:main:multi_index_complete}). However, once the direction $\vec w^\star_1$ is learned, it can be used to obtain order-one correlation with $\{\vec w^\star_2, \vec w^\star_3\}$ again in $T = O(d \log d)$ steps. While the latter example belongs to the class of ``CSQ'' staircase function depicted in \cite{abbe2021staircase}, novel ``SQ'' hierarchical mechanisms arise in our framework, such as for instance the function $f_{\mathrm{sq-stair}}^\star(\vec z) = \mathrm{He}_4(z_1) + \mathrm{sign}(z_1z_2z_3)$ (see Fig.~\ref{fig:main:staircase2}). 

The staircase picture in terms of generative exponent was discussed for AMP-type algorithm in \cite{troiani2024fundamental}. Our study indicate that this new class (they called grand staircase) is not restricted to AMP: We anticipate that grand staircase functions can be efficiently learned by two-layer networks when data reusing is permitted, whereas only standard staircase ones emerge  under single-pass SGD. Additional recent results by \citep{joshi2024complexity}  align with this perspective.

    

\section{Conclusions}
 We have pushed in this paper the boundaries of learning multi-index models in high-dimensional settings using gradient-type algorithms, and shown that, contrary to what was believed, they are very efficient in doing so. We argue in this manuscript with a simple analytically solvable model how to bypass the limitations of the CSQ framework: we process the same sample twice. Crucially, similar phenomenon appears also in standard SGD when training over multiple epochs, with the only difference of not seeing same the data consecutevely. We believe that the model proposed in this manuscript paves the way for the analysis of more realistic training scenario where \textit{correlation} in the data is taken into account, surpassing the limiting i.i.d. scenario.

\section*{Acknowledgements}
Upon completion of this work, we became aware of the interesting paper from \citep{lee2024neural} showing strong recovery with similar sample complexity, also by exploiting the reuse of training data. Our work was conducted independently and simultaneously. The authors would like to thank Nicolas Flammarion, Bruno Loureiro, Nati Srebro, Emanuele Troiani, and Lenka Zdeborov\'a for interesting discussions. This work was supported by the Swiss National Science Foundation under grant SNSF OperaGOST (grant number 200390).

\typeout{}

\bibliography{biblio}
\bibliographystyle{unsrtnat}

\newpage 

\appendix
\newpage 

\section{Theoretical framework}
\label{sec:app:framework}

\subsection{Difference inequalities for vector-valued processes}

We begin with a lemma that generalizes \cite[Proposition 4.1]{BenArous2021}. Let $\Phi: \R^d \to \R^d$ be a random function, and consider the process $(\vec{u}_t)_{t \geq 0} \in \mathbb{S}^{d-1}$ given by
\[ \vec{u}_{t+1} = \frac{\vec{u}_t + \gamma \Phi(\vec{u}_t)}{\norm{\vec{u}_t + \gamma \Phi(\vec{u}_t)}}. \]
We make the following assumptions:
\begin{assumption}\label{assump:app:initialization}
  The initial value of $\vec{u_0} \in \mathbb{S}^{d-1}$ is approximately isotropic, in the sense that there exists a constant $c > 0$ such that
  \[ \Ea{\vec{u_0}\vec{u_0}^\top} \succ \frac{c}d I\]
\end{assumption}
\begin{assumption}\label{assump:app:concentration}
  There exist constants $C, \iota > 0$ such that
  \begin{align}
    \sup_{\vec{u} \in \mathbb{S}^{d-1}}\, \norm{\Ea{\Phi(\vec u)\Phi(\vec u)^\top}} &\leq C \\
    \sup_{\vec{u} \in \mathbb{S}^{d-1}}\, \Ea{\norm{\Phi(\vec u)}^{4 + \iota}} &\leq C N^{\frac{4 + \iota}2}
  \end{align}
  Further, we have $\langle \vec{u}, \Phi(\vec{u}) \rangle = 0$ a.s. for any $\vec{u} \in \mathbb{S}^{d-1}$.
\end{assumption}

We fix a vector $\vec{v}$ and consider the random process $m_t(\vec{v}) := \langle \vec{u}_t, \vec{v} \rangle$. For any $s \geq 0$, we define the stopping times
\[ \tau_\theta^+(\vec{v}) = \inf \{t \geq 0 : m_t(\vec{v}) \geq \theta \} \quad \text{and} \quad \tau_\theta^-(\vec{v}) = \inf \{t \geq 0 : m_t(\vec{v}) \leq \theta \}  \]
\begin{proposition}\label{prop:app:diff_equation}
  Fix an $\alpha > 0$. There exist constants $c(\delta, \alpha), \kappa(\delta) > 0$ such that if
  \begin{equation}\label{eq:app:bound_gamma_T}
    \gamma \leq c(\delta, \alpha)d^{-1} \quad \text{and} \quad T \leq c(\delta, \alpha) \gamma^{-2} d^{-1},
  \end{equation}
  then, with probability $1 - \delta$, the following holds. The initialization $m_0(\vec{v})$ satisfies
  \begin{equation}
    |m_0(\vec{v})| \geq \frac{\kappa}{\sqrt{d}},
  \end{equation}
  and:
  \begin{itemize}
    \item if $m_0(v) > 0$, for any $t \leq T \wedge \tau_0^-(\vec{v})$,
      \begin{equation}\label{eq:app:diff_equation_positive}
        (1 - \alpha) m_0(\vec{v}) + \gamma\sum_{s=0}^{t-1} \langle \Psi^+(\vec{u}_s), \vec{v} \rangle \leq m_t(\vec{v}) \leq (1 + \alpha) m_0(\vec{v}) + \gamma\sum_{s=0}^{t-1} \langle \Psi(\vec{u}_s), \vec{v} \rangle,
      \end{equation}
    \item if $m_0(v) < 0$, for any $t \leq T \wedge \tau_0^+(\vec{v})$,
      \begin{equation}\label{eq:app:diff_equation_negative}
        (1 + \alpha) m_0(\vec{v}) + \gamma\sum_{s=0}^{t-1} \langle \Psi(\vec{u}_s), \vec{v} \rangle \leq m_t(\vec{v}) \leq (1 - \alpha) m_0(\vec{v}) + \gamma\sum_{s=0}^{t-1} \langle \Psi^-(\vec{u}_s), \vec{v} \rangle,
      \end{equation}
  \end{itemize}
  where $\Psi, \Psi^{\pm}$ is a deterministic function equal to
  \[ \Psi(\vec{u}) = \Ea{\Phi(\vec{u})} \quad \text{and} \quad \Psi^{\pm}(\vec{u}) = \Psi(\vec{u}) \mp \gamma \Ea{\norm{\Phi(\vec{u})}^2}\vec{u}.\]
\end{proposition}

\begin{proof}
  We assume for simplicity that $m_0(\vec{v}) > 0$; the case where $m_0(\vec{v}) > 0$ is treated identically. From Assumption~\ref{assump:app:initialization}, by a simple Markov inequality, there exists a constant $\kappa(\delta) > 0$ such that
  \begin{equation} \label{eq:app:large_init_event}
    \dP\left(|m_0(\vec{v})| \geq \frac{\kappa}{\sqrt{d}}\right) \geq 1 - \frac{\delta}3.
  \end{equation}
  The update equation for $m_t(\vec{v})$  reads
  \[ m_{t+1}(\vec{v}) = \frac{m_t(\vec{v}) + \gamma \langle \Phi(\vec{u}_t), \vec{v} \rangle}{\sqrt{1 + \gamma^2  \norm{\Phi(\vec{u}_t)}^2}}. \]
  We then have two cases:
  \begin{itemize}
    \item either $m_t(\vec{v}) + \gamma \langle \Phi(\vec{u}_t), \vec{v} \rangle < 0$, which implies that $t+1 = \tau_0^-(\vec{v})$ and there is nothing to prove,
    \item or we can use the inequality $\frac1{\sqrt{1+x}} \geq 1 - x$, which implies that
      \begin{align*}
        m_{t+1}(\vec{v}) &\geq (m_t(\vec{v}) + \gamma \langle \Phi(\vec{u}_t), \vec{v} \rangle)(1 - \gamma^2  \norm{\Phi(u_t)}^2) \\
        &= (1 - \gamma^2  \norm{\Phi(u_t)}^2)m_t + \gamma \langle \Phi(\vec{u}_t), \vec{v} \rangle - \gamma^3 \langle \Phi(\vec{u}_t), \vec{v} \rangle  \norm{\Phi(u_t)}^2
      \end{align*}
  \end{itemize}
  By recursion, deterministically,
  \[ m_t(\vec{v}) \geq m_0(\vec{v}) + \gamma\sum_{s = 0}^{t-1} \left(\langle \Phi(\vec{u}_s), \vec{v} \rangle - \gamma\norm{\Phi(\vec{u}_s)}^2\langle \vec{u}_s, \vec{v} \rangle\right) - \sum_{s=0}^{t-1}\gamma^3 \left|\langle \Phi(\vec{u}_s), \vec{v} \rangle\right|\norm{\Phi(\vec{u}_s)}^2 \]

  Since Assumption~\ref{assump:app:concentration} implies Assumption B in \cite{BenArous2021}, for a small enough choice of $c(\delta, \alpha)$, one has
  \[ \dP\left( \sum_{t=0}^T\gamma^3 \left|\langle \Phi(\vec{u}_t), \vec{v} \rangle\right| \norm{\Phi(\vec{u}_s)}^2 \leq \frac{\alpha\kappa}{2\sqrt{d}} \right) \geq 1 - \frac\delta6. \]

  Additionnally, the same assumption implies that
  \[ \operatorname{Var}(\langle \Phi(\vec{u}_s), \vec{v} \rangle)  \leq \Ea{\langle \Phi(\vec{u}_s), \vec{v} \rangle^2} \leq C_1 \quad \text{and} \quad \operatorname{Var}(\gamma\norm{\Phi(\vec{u}_s)}^2) \leq \gamma^2\Ea{\norm{\Phi(\vec{u}_s)}^4} \leq c(\delta)^2 C_1.\]
  We can then apply Doob's martingale inequality: there exists a constant $C(\delta)$ such that with probability at least $1 - \frac\delta6$
  \[ \sup_{t \leq T} \left| \gamma\sum_{s = 0}^t  \left( \langle \Phi(\vec{u}_s), \vec{v} \rangle - \gamma\norm{\Phi(\vec{u}_s)}^2\langle \vec{u}_s, \vec{v} \rangle - \Ea{\langle \Phi(\vec{u}_s), \vec{v} \rangle - \gamma\norm{\Phi(\vec{u}_s)}^2\langle \vec{u}_s, \vec{v} \rangle} \right) \right| \leq C(\delta)\gamma\sqrt{T}. \]
  The bound~\eqref{eq:app:bound_gamma_T} on $T$ implies that we can choose $c(\delta, \alpha)$ small enough so that
  \[ C(\delta)\gamma\sqrt{T} \leq \frac{C(\delta) \sqrt{c(\delta, \alpha)}}{\sqrt{d}} \leq \frac{\alpha\kappa}{2\sqrt{d}}.\]
  Combining all those inequalities with a union bound, under the event of Eq.~\eqref{eq:app:large_init_event},
  \begin{equation}\label{eq:app:diff_equation_lower}
    \dP\left( m_t(\vec{v}) \geq (1 - \alpha) m_0(\vec{v}) + \gamma\sum_{s=0}^{t-1} \langle \Psi^+(\vec{u}_s), \vec{v} \rangle \quad \forall t < T \wedge \tau_0^-(\vec{v})\right) \geq 1 - \frac\delta3,
  \end{equation}
  having used that $\langle \Psi^+(\vec{u}), \vec{v} \rangle = \Ea{\langle \Phi(\vec{u}), \vec{v} \rangle - \gamma\norm{\Phi(\vec{u})}^2\langle \vec{u}, \vec{v} \rangle}$.

  The upper bound follows similarly:  using this time the fact that for $x > 0$ we have $\frac1{\sqrt{1+x}} \leq 1$, we have under \eqref{eq:app:large_init_event}
  \begin{equation}\label{eq:app:diff_equation_upper}
    \dP\left( m_t(\vec{v}) \leq (1 + \alpha) m_0(\vec{v}) + \gamma\sum_{s=0}^{t-1} \langle \Psi(\vec{u}_s), \vec{v} \rangle \quad \forall t < T \wedge \tau_0^-(\vec{v})\right) \geq 1 - \frac\delta3.
  \end{equation}
  A union bound between the events of Eqs.~\eqref{eq:app:large_init_event}, \eqref{eq:app:diff_equation_lower}, \eqref{eq:app:diff_equation_upper} concludes the proof.
\end{proof}

\subsection{Low-dimensional sufficient statistics}

In \cite{BenArous2021}, the function $\Psi(\vec{u})$ can be written as $\Psi(\vec{u}) = \phi(\langle \vec{u}, \vec{v} \rangle)\vec{v}$. As a result, the bounds of Proposition~\ref{prop:app:diff_equation} are actually a closed function of $m_t(\vec{v})$. We generalize this phenomenon to a larger class:

\begin{assumption}\label{assump:app:low_dim}
  There exists a $q > 0$ such that the function $\Psi$ is $q$-approximately $k$-dimensional, in the sense that there exist a function $\psi: \R^k \to \R^k$ and an orthonormal matrix $W^\star \in \R^{k \times d}$ such that for any $\vec{u}\in\R^d$,
  \[ \norm{ W^\star \Psi(\vec{u}) - \psi(W^\star \vec{u})} \leq \left( \norm{W^\star \vec{u}} \vee \frac1{\sqrt{d}} \right) \cdot q. \]
\end{assumption}

By an abuse of notation, we shall sometimes view $\psi$ as a function from $V^\star$ to itself, where $V^\star$ is the span of the columns of $W^\star$. However, Assumption~\ref{assump:app:low_dim} ensures that all quantities related to $\psi$ are only depending on the (oftentimes fixed) dimension $k$ instead of $d$.

It turns out that in this case, we can control accurately the trajectory of $W^\star \vec{u}$, at least when $f$ is not too flat around $\vec{0}$. Let $\Pi^\star = (W^{\star})^{\top}W^\star$ be the projection operator on $V^\star$, and define the stopping time
\[ \tau_\theta = \min \{t \geq 0 : \norm{\Pi^\star \vec{u}} \geq \theta \}. \]

\begin{proposition}\label{prop:app:low_dim_linear}
  Assume that Condition \ref{assump:app:low_dim} holds with $q$ small enough, and that $\langle \vec{u}_0, \psi(\vec{0}) \rangle > 0$ (in particular, $\psi(\vec{0}) \neq 0$). For any $\delta, \epsilon > 0$, there exist constants $c(\delta, \epsilon), \eta(\epsilon)$ such that if
  \begin{equation}
    \gamma \leq c d^{-1} \quad \text{and} \quad T \leq c  \gamma^{-2} d^{-1},
  \end{equation}
  then
  \begin{equation}
    \dP\left(\norm{\Pi^\star \vec{u}_t - \gamma t \psi(\vec{0})} \leq \sqrt{k} \eta \epsilon \quad \forall t \leq T \wedge \tau_\eta \right) \geq 1 - \delta.
  \end{equation}
\end{proposition}

\begin{proof}
  Let $\vec{v}_1, \dots, \vec{v}_k$ be an orthonormal basis of $V^\star$, with $\vec{v}_1$ aligned with $\psi(\vec{0})$. By applying Proposition~\ref{prop:app:diff_equation} to $\delta' = \frac{\delta}{2k}$ and $\alpha = \frac12$, with probability $1 - \frac\delta2$, for any $i \in [k]$ and $t \leq T \wedge \tau^{\pm}(\vec{v}_i)$, $m_t(\vec{v}_i)$ satisfies the inequality~\eqref{eq:app:diff_equation_positive} or \eqref{eq:app:diff_equation_negative} (depending on the sign of $m_0(\vec{v_i}))$. In particular, for $\vec{v}_1 = \psi(\vec{0})$, since we assumed that the initial sign is positive, we have
  \begin{equation}\label{eq:app:u_psi_ineq}
    \frac12 \langle \vec{u}_0, \psi(\vec{0}) \rangle + \gamma\sum_{s=0}^{t-1} \langle \Psi^+(\vec{u}_s), \psi(\vec{0}) \rangle \leq \langle \vec{u}_t, \psi(\vec{0}) \rangle \leq  \frac32 \langle \vec{u}_0, \psi(\vec{0}) \rangle + \gamma\sum_{s=0}^{t-1} \langle \Psi(\vec{u}_s), \psi(\vec{0}) \rangle.
  \end{equation}

  Define $C(\delta)$ such that  $\langle \vec{u}_t, \psi(\vec{0}) \rangle \leq C(\delta)d^{-1/2}$ with probability at least $1 - \frac{\delta}{2k}$. By a Taylor expansion, there exists an $\eta(\epsilon) > 0$ such that for any $\vec{v} \in \R^d$ with $\norm{\vec{v}} \leq \eta$, $\norm{\psi(\vec{v}) - \psi(\vec{0})} \leq \epsilon'$, where $\epsilon'$ will be fixed later. With Assumption~\ref{assump:app:low_dim}, this implies that for any $\vec{u} \in \R^d$,
  \[ \norm{\psi(\vec{0})}^2 - (\epsilon' + \eta q) \norm{\psi(\vec{0})} \leq \langle \Psi(\vec{u}), \psi(\vec{0}) \rangle \leq \norm{\psi(\vec{0})}^2 + (\epsilon' + \eta q) \norm{\psi(\vec{0})}. \]
  Further, by Jensen's inequality, $\Ea{\norm{\Phi(\vec{u})}^2} \leq C_1 d$ so by choosing $c(\delta, \epsilon)$ small enough, we can ensure that
  \[ \langle \Psi^+(\vec{u}), \psi(\vec{0}) \rangle \geq \norm{\psi(\vec{0})}^2 - (2\epsilon + \eta q) \norm{\psi(\vec{0})} \]
  Plugging those estimates in \eqref{eq:app:u_psi_ineq}, we find that for $t \leq T \wedge \tau_0^-(\vec{v_1}) \wedge \tau_\eta$, and $d$ large enough so that
  \[ C(\delta)d^{-1/2}\leq \frac{\eta\epsilon}{3} \norm{\psi(\vec{0})}, \]
  we have
  \begin{equation}
    \gamma t \left(\norm{\psi(\vec{0})}^2 - (2\epsilon' + \eta q) \norm{\psi(\vec{0})}\right)  \leq \langle \vec{u}_t, \psi(\vec{0}) \rangle \leq \frac{\eta\epsilon}{2} \norm{\psi(\vec{0})} + \gamma t  \left(\norm{\psi(\vec{0})}^2 + (\epsilon' + \eta q) \norm{\psi(\vec{0})}\right).
  \end{equation}

  For small enough $\epsilon$ and $q$, the LHS of the above inequality is always at least $\gamma t \norm{\psi(0)}^2 / 2$, which implies that
  \begin{equation}\label{eq:app:bound_tau_eta}
    \tau_0^-(\vec{v}_1) \geq \tau_\eta \quad \text{and} \quad T \wedge \tau_\eta \leq \frac{2 \eta}{\gamma \norm{\psi(\vec{0})}}
  \end{equation}
  Choosing $\epsilon'$ and $\eta$ small enough so that
  \[ \frac{(2\epsilon' + \eta q)}{\norm{\psi(\vec{0})}^2} \leq \frac{\epsilon}{2}\]
  we ensure that for $t \leq T \wedge \tau_\eta$ we have
  \begin{equation}
    \left| \langle \vec{u_t} - \gamma t \psi(\vec{0}), \psi(\vec{0}) \rangle \right| \leq \epsilon \eta \norm{\psi(\vec{0})}.
  \end{equation}

  We now move on to the other $\vec{v}_i$. For $i \geq 2$, $\vec{v}_i$ is orthogonal to $\psi(\vec{0})$, and this time Assumption~\ref{assump:app:low_dim} implies
  \[ \left| \langle \Psi(\vec{u}), \vec{v}_i \rangle \right| \leq \epsilon + \eta q. \]
  Regardless of the sign of $m_0(\vec{v}_i)$, this implies that for $t \leq T \wedge \tau_\eta$,
  \[ |\langle \vec{u_t}, \vec{v}_i \rangle| \leq \frac32|\langle \vec{u}_0, \vec{v}_i \rangle| + \gamma t (\epsilon' + \eta q) \leq \frac{\eta\epsilon}{2} + \frac{2\eta(2\epsilon' + \eta q)}{\norm{\psi(\vec{0})}}, \]
  having used the bound~\eqref{eq:app:bound_tau_eta} on $\tau_\eta$ and bounding the initial value as above with probability $1- \frac\delta{2k}$. We can always choose $\eta, \epsilon'$ such that the above bound is at most $\epsilon$; since
  \[ \norm{\Pi_V \vec{u}_t - \gamma t \psi(\vec{0})}^2 = \frac{\langle \vec{u_t} - \gamma t \psi(\vec{0}), \psi(\vec{0}) \rangle^2}{\norm{\psi(\vec{0})}^2} + \sum_{i=2}^k \langle \vec{u_t}, \vec{v}_i \rangle^2 \leq k\epsilon^2, \]
  this concludes the proof.
\end{proof}

\begin{proposition}\label{prop:app:low_dim_exp}
  Assume that $\psi(\vec{0}) = \vec{0}$, and $H_\psi := d\psi(\vec{0}) \neq 0$ is symmetric. Let $\bar\lambda$ be the largest eigenvalue of $H_\psi$, and assume that $\bar\lambda > 0$.
  For any $\delta, \epsilon > 0$, there exist constants $c(\delta, \epsilon), \eta(\epsilon), c'(\delta, \epsilon), C(\delta, \epsilon)$ such that if
  \begin{equation}
    \gamma \leq c d^{-1} \quad \text{and} \quad T \leq c  \gamma^{-2} d^{-1},
  \end{equation}
  then
  \begin{equation}
    \dP\left(\norm{\Pi^\star \vec{u}_t} \geq c' \norm{\Pi^\star \vec{u}_0} e^{\gamma(\bar\lambda - Ck(\epsilon + q))t} \quad \forall t \leq T \wedge \tau_\eta \right) \geq 1 - \delta.
  \end{equation}
  Further, the mass of $\Pi^\star \vec{u}_t$ is concentrated on the top eigenspace of $H_\psi$: there exists a $\tilde\lambda < \bar\lambda - Ck(\epsilon + q)$ such that on the same event,
  \[ \norm{\Pi^\star \vec{u}_{\tau_\eta} - \bar\Pi^\star \vec{u}_t} \leq \norm{\Pi^\star \vec{u}_0} e^{\gamma \tilde\lambda t} \quad \forall t \leq T \wedge t_\eta, \]
  where $\bar\Pi^\star$ is the projection on the associated to $\bar\lambda$.
\end{proposition}

\begin{proof}
  Fix $(\vec{v}_1, \dots, \vec{v}_k)$ an orthonormal basis of $V^\star$, consisting of eigenvectors of $H_\psi$, and let $\lambda_i$ be the eigenvalue associated with $\vec{v}_i$. The largest eigenvalue of $H_\psi$ that differs from $\bar\lambda$ will be denoted $\underline\lambda$ (we can let $\underline\lambda = -\infty$ if $H_\psi = \bar\lambda I$).
  
  We choose an $\eta(\epsilon)$ such that for $\vec{v}\in \R^k$ such that $\norm{\vec{v}} \leq \eta$,
  \[ \norm{\psi(\vec{v}) - H_\psi \vec{v}} \leq \epsilon \norm{\vec{v}}. \]
  In particular, for any $i \in [k]$ and any vector $\vec{u}$, we have for small enough choice of $c(\delta, \epsilon)$
  \[ \langle \Psi^+(\vec{u}), \vec{v}_i \rangle \geq \lambda_i \langle \vec{v}_i, \vec{u} \rangle - (2\epsilon + q)\norm{\Pi^\star\vec{u}} \quad \langle \Psi(\vec{u}), \vec{v}_i \rangle \leq \lambda_i \langle \vec{v}_i, \vec{u} \rangle + (\epsilon + q)\norm{\Pi^\star\vec{u}} \]
  Applying Proposition~\ref{prop:app:diff_equation} to $\alpha = \epsilon$, and assuming that $m_0(\vec{v}_i) > 0$, we find that on an event with probability $1 - \delta/(2k)$, for $t \leq T \wedge \tau_0^+(\vec{v_i}) \wedge \tau_\eta$,
  \begin{align*}
     \langle \vec{v}_i, \vec{u}_t \rangle &\geq (1 - \epsilon)m_0(\vec{v}_i) + \gamma \sum_{s=0}^{t-1} \left(\lambda_i\langle \vec{v}_i, \vec{u}_s \rangle  - (2\epsilon + q)\norm{\Pi^\star\vec{u}_s}\right) \\ 
     \langle \vec{v}_i, \vec{u}_t \rangle &\leq (1 + \epsilon)m_0(\vec{v}_i) + \gamma \sum_{s=0}^{t-1} \left(\lambda_i\langle \vec{v}_i, \vec{u}_s \rangle + (\epsilon + q)\norm{\Pi^\star\vec{u}_s}\right)
  \end{align*}
  The same argument when $m_0(\vec{v}_i) < 0$, we obtain
  that with probability $1 - \delta/2$, for any $t \leq T \wedge \tau_\eta$ and $i \in [k]$:
  \begin{align}
    |\langle \vec{v}_i, \vec{u}_t \rangle| &\geq (1 - \epsilon)|\langle \vec{v}_i, \vec{u}_0 \rangle| + \gamma \sum_{s=0}^{t-1} \left(\lambda_i|\langle \vec{v}_i, \vec{u}_s \rangle|  - (2\epsilon + q)\norm{\Pi^\star\vec{u}_s}\right) \label{eq:app:lower_bound_exp}\\
    |\langle \vec{v}_i, \vec{u}_t \rangle| &\leq (1 + \epsilon)|\langle \vec{v}_i, \vec{u}_0 \rangle| + \gamma \sum_{s=0}^{t-1} \left(\lambda_i|\langle \vec{v}_i, \vec{u}_s \rangle| + (\epsilon + q)\norm{\Pi^\star\vec{u}_s}\right)\label{eq:app:upper_bound_exp}
  \end{align}
Define $\bar\Pi^\star$ the projection on the eigenspace associated to $\bar\lambda$, and $\underline\Pi^\star$ its orthogonal projection. Summing the upper bound of \eqref{eq:app:upper_bound_exp} on all $v_i$ with $\lambda_i \leq \underline\lambda$,
\[ \norm{\underline\Pi^\star \vec{u}_t}_1 \leq (1 + \epsilon)\norm{\underline\Pi^\star \vec{u}_0}_1 + \gamma (\underline\lambda + k(\epsilon + q)) \sum_{s=0}^{t-1} \norm{\Pi^\star \vec{u}_s}_1, \]
  where the $\ell^1$ norm $\norm{\:\cdot\:}_1$ in $V^\star$ is computed w.r.t the basis $(\vec{v}_1, \dots, \vec{v}_k)$, and we used the norm bound $\norm{\Pi^\star \vec{u}_s} \leq \norm{\Pi^\star \vec{u}_s}_1$.

  Call $\cP(t_0)$ the following condition:
  \begin{equation}\label{eq:app:rec_condition}
      \norm{\underline\Pi^\star \vec{u}_t}_1 \leq C\norm{\bar\Pi^\star \vec{u}_t}_1 \quad \text{for all} \  t < t_0, ,
  \end{equation} where $C$ will be fixed later, and assume that $\cP(t_0)$ holds for some $t_0 > 0$.
  Using the discrete Grönwall inequality, the following holds: for any $t \leq T \wedge \tau_\eta \wedge t_0$,
  \begin{equation}\label{eq:app:bad_upper_bound}
      \norm{\underline\Pi^\star \vec{u}_t}_1 \leq (1 + \epsilon)\norm{\underline\Pi^\star \vec{u}_0}_1 e^{\gamma (\underline\lambda + (1+C)k(\epsilon + q))t}.
  \end{equation} 
   We now sum the lower bound of \eqref{eq:app:lower_bound_exp} on all $v_i$ with $\lambda_i = \bar\lambda$, to obtain
    \[ \norm{\bar\Pi^\star\vec{u}_t}_1 \geq (1 - \epsilon) \norm{\bar\Pi^\star \vec{u}_0}_1 + \gamma \sum_{s=0}^{t-1} (\bar\lambda - k(2\epsilon + q)) \norm{\bar\Pi^\star \vec{u}_s}_1 - k(2\epsilon + q)\norm{\underline\Pi^\star \vec{u}_s}_1. \]
      above bound implies that for all $t \leq t_0$,
     \[ \norm{\bar\Pi^\star\vec{u}_t}_1 \geq (1 - \epsilon) \norm{\bar\Pi^\star \vec{u}_0}_1 + \gamma \sum_{s=0}^{t-1} (\bar\lambda - (1+C)k(2\epsilon + q)) \norm{\bar\Pi^\star \vec{u}_s}_1. \]
     By the discrete Grönwall inequality,
     \begin{align} 
        \norm{\bar\Pi^\star\vec{u}_t}_1 &\geq (1 - \epsilon) \norm{\bar\Pi^\star \vec{u}_0}_1 \left(1 +  \gamma(\bar\lambda - (1+C)k(2\epsilon + q)) \right)^t \nonumber\\ 
        &\geq (1 - \epsilon) \norm{\bar\Pi^\star \vec{u}_0}_1 \exp\left(\gamma(\bar\lambda - (1+C)k(2\epsilon + q))t - O(\gamma^2 T) \right) \nonumber\\
        &\geq (1 - 2\epsilon) \norm{\bar\Pi^\star \vec{u}_0}_1 \exp\left(\gamma(\bar\lambda - (1+C)k(2\epsilon + q))t\right), \label{eq:app:good_lower_bound}
     \end{align}
     where the last inequality holds for large enough $d$ because of the bound \eqref{eq:app:bound_gamma_T}.
     Now, fix a constant $r(\delta)$ such that with probability $1 - \delta/2$, $\norm{\bar\Pi^\star \vec{u}_0}_1 \leq r(\delta)\norm{\underline\Pi^\star \vec{u}_0}_1$. Then condition $\cP(1)$ holds as long as $C > r(\delta)$, and $\cP(t_0)$ implies $\cP(t_0+1)$ whenever the lower bound of eq. \eqref{eq:app:good_lower_bound} is higher than the upper bound of eq. \eqref{eq:app:bad_upper_bound} for $t \leq t_0$. In particular, this holds whenever
     \[ C \geq \frac{1 + \epsilon}{1 - 2\epsilon}r(\delta) \quad \text{and} \quad \bar\lambda - (1+C)k(2\epsilon + q) \geq  \underline\lambda + (1+C)k(\epsilon + q).\]
     The above condition can be satisfied if and only if
     \[ \frac{1 + \epsilon}{1 - 2\epsilon}r(\delta) < \frac{\bar\lambda - \underline\lambda}{3\epsilon + 2q} - 1, \]
     which can always be satisfied when $\epsilon, q$ are low enough. The proposition ensues from the remark that $\norm{\bar\Pi^\star\vec{u}_t} \leq \norm{\Pi^\star \vec{u}} \leq (1+C)\norm{\bar\Pi^\star \vec{u}_t}$ for any $t \leq T \wedge \tau_\eta$.
\end{proof}

\subsection{Bounding the hitting times}

It finally remains to show how the bounds of Propositions \ref{prop:app:low_dim_linear} and \ref{prop:app:low_dim_exp} allow us to bound the hitting times $\tau_\eta$. We summarize those in the following proposition:
\begin{proposition}\label{prop:app:hitting_times}
    Let $\vec{u}$ be a process that satisfies Assumptions \ref{assump:app:initialization}-\ref{assump:app:low_dim} for low enough $q$. Then for any $\delta, \epsilon > 0$, there exist constants $\gamma_0(\delta, \epsilon), \eta(\epsilon), \kappa(\delta), K(\delta), C(\delta, \epsilon)$ such that, with probability $1 - \delta$:
    \begin{enumerate}
        \item if $\psi(\vec{0}) \neq \vec{0}$ and $\langle \vec{u}_0, \psi(\vec{0}) \rangle > 0$, letting $\gamma = \gamma_0 d^{-1}$,
        \[ \frac{\eta(1 - \sqrt{k}\epsilon)}{\gamma \norm{\psi(\vec{0})}} \leq \tau_\eta \leq \frac{\eta(1 + \sqrt{k}\epsilon)}{\gamma \norm{\psi(\vec{0})}} \]
        \item if $\psi(\vec{0}) = 0$ and $d\psi(\vec{0})$ is symmetric and has a positive maximum eigenvalue $\bar\lambda$, then letting $\gamma = \gamma_0 (d\log(d))^{-1}$
        \[ \frac{\log\left(\frac{\eta\sqrt{d}}{(1+\epsilon)K}\right)}{\gamma(\bar\lambda + k(\epsilon + q))} \leq \tau_\eta \leq \frac{\log\left(\frac{\eta\sqrt{d}}{\kappa c'}\right)}{\gamma(\bar\lambda - Ck(\epsilon + q))}\]
    \end{enumerate}
    Further, the directions of $\Pi^\star \vec{u}_{\tau_\eta}$ concentrate, in the following sense:
    \begin{enumerate}
        \item if $\psi(\vec{0}) \neq \vec{0}$,
        \[ \norm{\frac{\Pi^\star \vec{u}_{\tau_\eta}}{\eta} -  \frac{\psi(\vec{0})}{\norm{\psi(\vec{0})}}} \leq 2\sqrt{k}\epsilon \]
        \item if $\psi(\vec{0}) = \vec{0}$, for some $\beta < 1$,
        \[ \norm{\Pi^\star \vec{u}_{\tau_\eta} - \bar \Pi^\star \vec{u}_{\tau_\eta}} \leq C(\delta, \epsilon) d^{-\frac{1 - \beta}{2}},\]
        where $\bar\Pi^\star$ is the projection on the top eigenspace of $d\psi(\vec{0})$.
    \end{enumerate}
\end{proposition}

\begin{proof}
    The proof for the linear case is almost immediate. From Proposition \ref{prop:app:low_dim_linear}, with probability $1 - \delta$, we have for all $t \leq T \wedge \tau_\eta$
    \[ \gamma t \norm{\psi(\vec{0})} - \sqrt{k}\eta\epsilon \leq \norm{\Pi^\star\vec{u}_t} \leq \gamma t \norm{\psi(\vec{0})} + \sqrt{k}\eta\epsilon \]
    Define
    \[ \bar t_\epsilon = \frac{\eta(1 + \sqrt{k}\epsilon)}{\gamma \norm{\psi(\vec{0})}} \quad \text{and} \quad \underline t_\epsilon = \frac{\eta(1 - \sqrt{k}\epsilon)}{\gamma \norm{\psi(\vec{0})}}. \]
    For $t \leq \underline t_\epsilon$, the upper bound on $\norm{\Pi^\star \vec{u}_t}$ is lower than $\eta$, which directly means that $\tau_\eta \geq \underline t_\epsilon$. The lower bound is higher than $\eta$ for $t = \bar t_\epsilon$, which implies that $\tau_\eta \leq \bar t_\epsilon$ as long as $T \geq \bar t_\epsilon$. Such a $T$ satisfies the bound \eqref{eq:app:bound_gamma_T} as long as
    \[ \frac{\eta(1 + \sqrt{k}\epsilon)}{\gamma \norm{\psi(\vec{0})}} \leq c(\delta, \epsilon) \gamma^{-2}d^{-1}, \]
    which is equivalent to
    \[ \gamma \leq \frac{c(\delta,\epsilon)\norm{\psi(\vec{0})}}{\eta(1 + \sqrt{k}\epsilon)}\, d^{-1},\]
    which can always be ensured by decreasing $\gamma_0(\delta, \epsilon)$. Finally, we use the bound
    \[ \norm{\frac{\vec{x}}{\norm{\vec{x}}} - \frac{\vec{y}}{\norm{\vec{y}}}} \leq 2 \frac{\norm{\vec{x} - \vec{y}}}{\norm{\vec{x}}} \]
    to conclude that
    \[ \norm{\frac{\Pi^\star \vec{u}_{\tau_\eta}}{\eta} - \frac{\psi(\vec{0})}{\norm{\psi(\vec{0})}}} \leq 2\sqrt{k}\epsilon. \]

    We now move on to the quadratic case. By Proposition \ref{prop:app:low_dim_exp} and the same reasoning as in the proof of \eqref{eq:app:bad_upper_bound}, we have with probability $1 - \delta$, for all $t \leq T \wedge \tau_\eta$,
    \[ c' \norm{\Pi^\star \vec{u}_0} e^{\gamma(\bar\lambda - Ck(\epsilon + q))t} \leq \norm{\Pi^\star \vec{u}_t} \leq (1 + \epsilon)\norm{\Pi^\star \vec{u}_0} e^{\gamma(\bar\lambda + k(\epsilon + q))t} \]
    Upon redefining $\delta$, we can assume that the ``good'' event implies that
    \[ \frac{\kappa(\delta)}{\sqrt{d}} \leq \norm{\Pi^\star \vec{u}_0} \leq \frac{K(\delta)}{\sqrt{d}} \]
    for constants $k, K$. As in the linear case, the upper bound above implies that
    \[ \tau_\eta \geq \frac{\log\left(\frac{\eta\sqrt{d}}{(1+\epsilon)K}\right)}{\gamma(\bar\lambda + k(\epsilon + q))}, \]
    while the lower bound implies that
    \[ T \wedge \tau_\eta \leq \frac{\log\left(\frac{\eta\sqrt{d}}{\kappa c'}\right)}{\gamma(\bar\lambda - Ck(\epsilon + q))} =: \bar t_\epsilon'. \]
    It therefore remains to show that $\bar t_\epsilon' \leq T$; for large enough $d$ and small enough $\epsilon, q$, we can write $\bar t_\epsilon' \leq C'(\delta, \epsilon) \gamma^{-1} \log(d)$, and
    \[ C'(\delta, \epsilon) \gamma^{-1} \log(d) \leq c\gamma^{-2}d^{-1} \Leftrightarrow \gamma \leq c(C' d \log(d))^{-1}.\]
    This condition can again be ensured by decreasing $\gamma_0(\delta, \epsilon)$ as needed. Finally, let 
    \[ \beta = \frac{\tilde\lambda}{\bar\lambda - Ck(\epsilon+q)}, \]
    where $\tilde \lambda$ is the one of Proposition~\ref{prop:app:low_dim_exp}. Then $\beta < 1$ and we have at $t = \tau_\eta$
    \[ \norm{\vec{u}_{\tau_\eta} - \bar\Pi^\star \vec{u}_{\tau_\eta}} \leq \frac{K}{\sqrt{d}} \exp\left(\beta \log\left(\frac{\eta\sqrt{d}}{\kappa c'}\right)\right) \leq C''(\delta, \epsilon) d^{\frac{1 - \beta}{2}}. \]
\end{proof}

\section{Proof of the Main Results}
\label{sec:app:proofs}

\subsection{Proof of Theorem \ref{thm:single_index_recovery}}

We formalize here the proof sketch of Theorem \ref{thm:single_index_recovery}.  We track the dynamics of $\vec{w}_t$ through the following sufficient statistic:
\begin{equation}
  m_t := \langle \vec{w}_t, \vec{w}^\star \rangle.
\end{equation}
The dynamics of $m_t$ fall under the framework of Section~\ref{sec:app:framework}, by letting
\[ \Phi(\vec{w}) = -\nabla_{\vec w}^\bot\cL\left(f(\vec z; \tilde{\vec{w}}(\rho), a_0),  y \right). \]
Indeed, for the spherical gradient,
\begin{align*}
  \nabla_{\vec{w}}^\bot \cL(f(\vec{z}; \vec{w}, a_0), y) = a_0 y \sigma'(\langle \vec{w}, z \rangle) \cdot \left(\vec{z} - \langle \vec{z}, \vec{w} \rangle \vec{w} \right).
\end{align*},
hence plugging in the expression for $\tilde{\vec{w}}(\rho)$
\begin{equation}\label{eq:app:phi_expression}
    \Phi(\vec{w}) = h^\star(\langle \vec{w}^\star, \vec{z} \rangle) \sigma'\left(\langle \vec{w}, \vec{z} \rangle + \rho h^\star(\langle \vec{w}^\star, \vec{z} \rangle) \sigma'(\langle \vec{w}, \vec{z} \rangle) \cdot \norm{\vec{z}}^2\right) \left(\vec{z} - \langle \vec{z}, \vec{w} \rangle \vec{w} \right).
\end{equation} 
Assumptions  \ref{assump:poly_growth} and \ref{assump:init} on $\sigma$ and $h^\star$ then directly imply Assumptions \ref{assump:app:initialization} and \ref{assump:app:concentration} in the Appendix. As a result, we only need to study the quantity $\Psi(\vec{w}) = \Ea{\Phi(\vec{w})}$.
We show the following:
\begin{lemma}\label{lem:app:psi_low_rank}
    The function $\Psi$ is approximately one-dimensional: there exists a function $\phi$ such that for any $\vec{w} \in \mathbb{S}^{d-1}$, we have
    \[ |\langle \vec{w}^\star, \Psi(\vec{w}) \rangle- a_0 \phi(\vec{w}, \vec{w}^\star)| \leq \frac{C}d \]
\end{lemma}

\begin{proof}
Define the following quantities:
\begin{align}
  &\Xi(x, x^\star, x^\bot) = h^\star(x^\star) \sigma'(x + a_0\rho_0 \sigma'(x) h^\star(x^\star)) x^\bot \\
  &\phi(m) = \E[\Xi(x, x^\star, x^\bot)] \  \text{for} \  \dbinom{x}{x^\star} \sim  \cN\left( \vec{0},
    \begin{pmatrix}
      1 & m \\
      m & 1
  \end{pmatrix} \right) \  \text{and} \  x^\bot = x^\star - m x.\label{eq:app:def_phi}
\end{align}

  For simplicity, we let $u = \langle \vec{w}, \vec{z} \rangle + a_0\rho_0 h^\star(\langle \vec{w}^\star, \vec{z} \rangle)$. Doing a Taylor expansion of $\sigma'$ in the second order, we have
  \begin{align*}
    \sigma'\left(\langle \vec{w}, \vec{z} \rangle + \rho h^\star(\langle \vec{w}^\star, \vec{z} \rangle) \sigma'(\langle \vec{w}, \vec{z} \rangle) \cdot \norm{\vec{z}}^2\right) = &\sigma'(u)
    + \sigma''(u) \rho h^\star(\langle \vec{w}^\star, \vec{z} \rangle) \sigma'(\langle \vec{w}, \vec{z} \rangle) \cdot  \left(\norm{\vec{z}}^2 - d \right) \\
    &+ R(u) \rho^2 h^\star(\langle \vec{w}^\star, \vec{z} \rangle)^2 \sigma'(\langle \vec{w}, \vec{z} \rangle)^2 \cdot  \left(\norm{\vec{z}}^2 - d \right)^2 ,
  \end{align*}
  where $R(u)$ is the Taylor remainder of $\sigma'$ at second order around $u$.

  The distribution of the triplet $(\langle \vec{w}, \vec{z} \rangle, \langle \vec{w^\star}, \vec{z} \rangle, \langle \vec{w^\star}, \vec{z}^\bot \rangle)$ is the same as the one of $(x, x^\star, x^\bot)$, hence
  \begin{align*}
    \Ea{\langle \Phi(\vec{z}), \vec{w}^\star \rangle} - a_0\phi(m) = &\underbrace{\rho\Ea{\sigma''(u)  h^\star(\langle \vec{w}^\star, \vec{z} \rangle)^2 \sigma'(\langle \vec{w}, \vec{z} \rangle) \cdot  \left(\norm{\vec{z}}^2 - d \right) \langle \vec{w}^\star, \vec{z}^\bot \rangle}}_{(1)} \\
    &+ \underbrace{\rho^2 \Ea{R(u) h^\star(\langle \vec{w}^\star, \vec{z} \rangle)^3 \sigma'(\langle \vec{w}, \vec{z} \rangle)^2 \cdot  \left(\norm{\vec{z}}^2 - d \right)^2  \langle \vec{w}^\star, \vec{z}^\bot \rangle}}_{(2)}.
  \end{align*}
  We first bound (2). From the Hölder inequality (since $\frac12 + \frac18 + \frac18 + \frac18 + \frac18 = 1$),
  \begin{align*}
    (2) &\leq \rho^2 \Ea{R(u)^8}^{\frac18}\Ea{h^\star(x^\star)^8}^{\frac18}\Ea{\sigma'(x)^8}^{\frac18}\Ea{\langle \vec{w}^\star, \vec{z}^\bot \rangle^8}^{\frac18}\Ea{\left(\norm{\vec{z}}^2 - d \right)^4}^{\frac12}
  \end{align*}
  Since $\sigma$ is three times differentiable almost everywhere, and from the polynomial bounds of Assumption~\ref{assump:derivatives}, the first four expectations are bounded by an absolute constant. The last term is the centered fourth moment of a $\chi^2(d)$ distribution, which is equal to $12d(d+4)$. As a result, using that $\rho^2 = O(d^{-2})$,
  \[ (2) \leq Cd^{-1}. \]
  For the first term, we let $\vec{z} = \vec{z}_\parallel + \vec{z}'$, where $\vec{z}_\parallel$ is the projection of $\vec{z}$ on the span of $\vec{w}$ and $\vec{w}^\star$. We can then write
  \begin{align*} 
    (1)& = \underbrace{\Ea{\sigma''(u) \rho  h^\star(\langle \vec{w}^\star, \vec{z} \rangle)^2 \sigma'(\langle \vec{w}, \vec{z} \rangle) \cdot  \norm{\vec{z}_\parallel}^2 \cdot \langle \vec{w}^\star, \vec{z}^\bot \rangle}}_{(1')} \\
    &\quad+ \underbrace{\Ea{\sigma''(u) \rho  h^\star(\langle \vec{w}^\star, \vec{z} \rangle)^2 \sigma'(\langle \vec{w}, \vec{z} \rangle) \cdot  \left(\norm{\vec{z'}}^2 - d \right) \langle \vec{w}^\star, \vec{z}^\bot \rangle}}_{(1'')}.
  \end{align*}
  Up to the factor of $\rho$, the term $(1')$ is the expectation of a polynomially bounded function of a $2$-dimensional Gaussian variable, hence
  \[ (1') \leq C \rho \leq C' d^{-1}. \]
  In $(1'')$, the vector $\vec{z'}$ is independent from $\langle \vec{w}, \vec{z} \rangle$ and $\langle \vec{w}^\star, \vec{z} \rangle$, so the expectation factorizes as
  \begin{align*}
    |(1'')| &= \left|\Ea{\sigma''(u) \rho  h^\star(\langle \vec{w}^\star, \vec{z} \rangle)^2 \sigma'(\langle \vec{w}, \vec{z} \rangle) \langle \vec{w}^\star, \vec{z}^\bot \rangle }\cdot  \Ea{\left(\norm{\vec{z'}}^2 - d \right) }\right| \\
    &\leq C \rho |d - 2 - d| \leq C' d^{-1},
  \end{align*}
  where we used that $\vec{z}'$ is a $(d-2)$-dimensional Gaussian variable. This concludes the proof.
\end{proof}

Lemma \ref{lem:app:psi_low_rank} implies that Assumption~\ref{assump:app:low_dim} holds for $\Psi$ and any choice of $q$, as long as $d$ is large enough.  As a result, the proof of Theorem~\ref{thm:single_index_recovery} hinges on the following facts about $\phi$:
\begin{proposition}\label{prop:app:phi_phip_nnz}
  For almost any choice of $\rho_0$ (w.r.t to the Lebesgue measure),
  \begin{itemize}
    \item if $\ell_p^\star = 1$, then $\phi(0) \neq 0$;
    \item if $\ell_p^\star = 2$, then $\phi'(0) \neq 0$.
  \end{itemize}
\end{proposition}

\begin{proof}[of Theorem \ref{thm:single_index_recovery}]
    Our goal is to apply Proposition~\ref{prop:app:hitting_times} to the function $\phi$. If $\ell_p^\star = 1$, then $\phi(0) \neq 0$ by the previous lemma, and with a uniform initialization of $\vec{w}^\star$ we have $m_0 \phi(0) > 0$ with probability $1/2$.  Combining this with the $(1- \delta)$ probability event of Proposition \ref{prop:app:hitting_times} for small enough $\epsilon$, with probability $1 - \delta$, we have $t_\eta \leq C(\delta)d^{-1}$ with probability $1/2 - \delta$.
    We now move to the proof of the case where $\ell_p^\star =1$. In this case, Proposition~\ref{prop:app:hitting_times} implies the desired result as soon as $\phi'(0) > 0$ (which corresponds in the $k = 1$ setting to the only eigenvalue of $d\phi(0)$ being positive). Since the Lebesgue measure is invariant w.r.t multiplication by $\pm 1$, the factor $a_0 \rho_0$ in the definition of $\phi'(0)$ has the same distribution as $\rho_0$, hence we can write $\phi'(0) = a'_0 |\phi'(0)|$ where $a'_0$ is also uniform in $\{-1, 1\}$. The event $\phi'(0) > 0$ therefore also has probability $1/2$, and we conclude as before.
\end{proof}

\subsubsection*{Analysis of $\phi$: proof of Proposition~\ref{prop:app:phi_phip_nnz}}

The main ingredient here is to write $\phi(m) = \phi(m; \rho_0)$ and differentiate w.r.t $\rho_0$. For simplicity, we take $a_0 = 1$ in the definition of $\phi$. We show the following:
\begin{lemma}\label{lem:app:phi_diff}
  Under Assumption \ref{assump:derivatives}, the function $\phi(m; \rho_0)$ is analytic in $\rho_0$. Further, we have
  \begin{equation}
    \left.\frac{\partial^k \phi(m; \rho_0)}{\partial \rho_0^k}\right|_{\rho_0 = 0} = \Ea{h^\star(x^\star)^{k+1} x^\bot \sigma^{(k+1)}(x) \sigma'(x)^k}
  \end{equation}
  where $x, x^\star$ follow the same distribution as in eq.~\eqref{eq:app:def_phi}.
\end{lemma}

\begin{proof}
  The first statement stems from the analyticity of $\sigma$, and the fact that it is a property conserved through integration. For the second, we differentiate inside the expectation, to find
  \[ \frac{\partial \phi}{\partial \rho_0} = \Ea{\sigma'(x)h^\star(x^\star)^2 \sigma''(x + \rho_0 \sigma'(x)h^\star(x^\star)) x^\bot}; \]
  the result follows by induction and taking $\rho_0 = 0$ at the end.
\end{proof}

We define the following quantities related to Assumption~\ref{assump:init}:
\begin{align*}
  u_{0}^{(k)} &= \Ea{\sigma^{(k)}(x) \sigma'(x)^{k-1}}  &u_{1}^{(k)} &= \Ea{x \sigma^{(k)}(x) \sigma'(x)^{k-1}} \\
  v_{0}^{(k)} &= \Ea{x^\star h^\star(x^\star)^k} & v_{1}^{(k)} &= \Ea{[(x^\star)^2-1] h^\star(x^\star)^k} \\
  \Delta^{(k)} &= \Ea{h^\star(x^\star)^k}
\end{align*}
Then, by Proposition 11.37 from \cite{odonnell_2014_analysis}, and using $x^\bot = x^\star - mx$:
\begin{align*}
  \left.\frac{\partial^k \phi(m; \rho_0)}{\partial \rho_0^k}\right|_{\rho_0 = 0} &= \Ea{h^\star(x^\star)^{k+1} x^\star \sigma^{(k+1)}(x) \sigma'(x)^k} - m \cdot \Ea{h^\star(x^\star)^{k+1} x \sigma^{(k+1)}(x) \sigma'(x)^k} \\
  &= \left(u_{0}^{(k)}v_{0}^{(k)} + u_{1}^{(k)}(v_{1}^{(k)} + \Delta^{(k)})\cdot m + o(m)\right) - m \left(u_1^{(k)} \Delta^{(k)} + o(1)\right) \\
  &= u_{0}^{(k)}v_{0}^{(k)} + u_{1}^{(k)}v_{1}^{(k)}\cdot m + o(m)
\end{align*}
This allows us to show Proposition \ref{prop:app:phi_phip_nnz}:
\begin{proof}
  Assume first that $\ell_p^\star = 1$. Then $\phi(0; \rho_0)$ is an analytic function of $\rho_0$ with
  \[ \left.\frac{\partial^k \phi(0; \rho_0)}{\partial \rho_0^k}\right|_{\rho_0 = 0} = u_{0}^{(k)}v_{0}^{(k)}. \]
  Since $\ell_p^\star = 1$, there exists a $k \in \mathbb{N}$ such that $v_{0}^{(k)} \neq 0$, and by Assumption~\ref{assump:derivatives} the coefficient $u_{0}^{(k)}$ is nonzero for every $k$. As a result, $\phi(0;\rho_0)$ is an analytic and non-identically zero function of $\rho_0$, so it is non-zero for almost every choice of $\rho_0$, as requested. The case $k_p = 2$ is done in a similar way, noting that this time
  \[ \left.\frac{\partial^k \phi'(0; \rho_0)}{\partial \rho_0^k}\right|_{\rho_0 = 0} = u_{1}^{(k)}v_{1}^{(k)}. \]
\end{proof}

\subsection{Proof of Theorem \ref{thm:multi_index_recovery}}

The proof of Theorem~\ref{thm:multi_index_recovery} follows in the same way as the one of Theorem~\ref{thm:single_index_recovery}. Recall the definition of $\Phi$ in eq.~\eqref{eq:app:phi_expression}, and define
\begin{align}
    &\Xi(z, \vec{z}^\star, \vec{z}^\bot) = h^\star(\vec{z}^\star) \sigma'(z + a_0\rho_0\sigma'(z)h^\star(\vec{z}^\star)) \vec{z}^\bot \\
    &\psi(\vec{p}) = \Ea{\Xi(z, \vec{z}^\star, \vec{z}^\bot)} \  \text{for} \  \dbinom{z}{\vec{z}^\star} \sim  \cN\left( \vec{0},
    \begin{pmatrix}
      1 & \vec{p} \\
      \vec{p} & I_k
  \end{pmatrix} \right) \  \text{and} \  \vec{z}^\bot = \vec{z}^\star - z\vec{p}.\label{eq:app:def_psi}
\end{align}
Again, the variables $(z, \vec{z}^\star)$ correspond to $\langle \vec{w}, \vec{z} \rangle, W^\star\vec{z}$, and the vector $\vec{p}$ to $W^\star\vec{w}$. As before, the only difference between $\Ea{\Phi(\vec{w})}$ and $\psi(W^\star \vec{w})$ lies in replacing $\norm{\vec{z}}$ with $d$, and the proof of Lemma~\ref{lem:app:psi_low_rank} implies the following: for any $\vec{w}\in \mathbb{S}^{d-1}$,
\[ \norm{W^\star\Psi(\vec{w}) - \psi(W^\star \vec{w})} \leq \frac Cd. \]
As a result, the function $\Psi$ is approximately $k$-dimensional for any choice of $q$ as $d \to \infty$. 

We can view the vector $\psi(\vec{p})$ as a random variable depending on the choice of $a_0, \rho_0$. Since $a_0$ is a sign, we will only consider the dependency of $\psi$ in $\rho_0$. As before, let $u_i^{(k)}$ be the $i$-th Hermite coefficient of $z \mapsto \sigma^{(k+1)}(z) \sigma'(z)^k$, and $C_i^{(k)}$ the $i$-th Hermite tensor of $h^\star$ \citep{grad_1949_note} -- in particular, $C_i^{(k)}$ is a tensor of order $i$.
We have the analogous result to Lemma~\ref{lem:app:phi_diff}:

\begin{lemma}
    Under Assumption~\ref{assump:derivatives}, the function $\psi(\vec{p}; \rho_0)$ is analytic in $\rho_0$, with
    \[  \left.\frac{\partial^k \psi(\vec{p}; \rho_0)}{\partial \rho_0^k}\right|_{\rho_0 = 0} = u_0^{(k)}C_1^{(k)} + u_1^{(k)}C_2^{(k)} \vec{p} + o(\norm{\vec{p}}). \]
\end{lemma}

\begin{proof}
    As in Lemma~\ref{lem:app:phi_diff}, we can write $\psi(\vec{0}) = \psi(\vec{0}, \rho_0)$, with
    \[  \left.\frac{\partial^k \psi(\vec{p}; \rho_0)}{\partial \rho_0^k}\right|_{\rho_0 = 0} = \Ea{h^\star(\vec{z}^\star)^{k+1} \vec{z}^\bot \sigma^{(k+1)}(z) \sigma'(z)^k}. \]
     From \cite[Lemma 19]{dandi2023twolayer}, we have
    \[ \Ea{h^\star(\vec{z}^\star)^{k+1} \vec{z}^\star \sigma^{(k+1)}(z) \sigma'(z)^k} = \sum_{i=0}^\infty u_{i+1}^{(k)} \langle \vec{p}^{\otimes i}, C_i^{(k)} \rangle \vec{p} + \sum_{i = 0}^\infty u_{i}^{(k)} C_{i+1}^{(k)} \times_{1, \dots, i} \vec{p}^{\otimes i}. \]
    Since $\vec{z}^\bot = \vec{z} - z\vec{p}$, we find
    \begin{equation}
        \left.\frac{\partial^k \psi(\vec{p}; \rho_0)}{\partial \rho_0^k}\right|_{\rho_0 = 0} = \sum_{i=0}^{\infty} u_i^{(k)} \left( C_{i+1}^{(k)} \times_{1, \dots, i} \vec{p}^{\otimes i} - \langle C_{i+1}^{(k)}, \vec{p}^{\otimes (i+1)} \rangle \vec{p} \right).
    \end{equation}
    For $i \geq 0$, we have $\norm{C_{i+1}^{(k)} \times_{1, \dots, i} \vec{p}^{\otimes i}} = O\left(\norm{\vec{p}}^i\right)$ and $\langle C_{i}^{(k)}, \vec{p}^{\otimes (i)} \rangle = O\left(\norm{\vec{p}}^i\right)$. The only terms in the sum which are not of order $\norm{\vec{p}}^2$ are the first two terms in the left sum, which yields the statement of the Lemma.
\end{proof}

We are now ready to show the first part of Theorem~\ref{thm:multi_index_recovery}. We begin with a characterization of $U^\star$ whenever $\ell_p^\star = 1$:
\begin{lemma}
    If $\ell_p^\star = 1$, then
    \[ U^\star = \operatorname{span}\left(\left\{C_1^{(k)}\right\}_{k \geq 1}\right). \]
    If $\ell_p^\star = 2$, then
    \[ U^\star = \bigoplus_{k\geq 1}\mathrm{Im}\left(C_2^{(k)}\right). \]
\end{lemma}

\begin{proof}
    By the properties of the Hermite tensors, we have for any $v \in V^\star$ and any polynomial $p$
    \[ \Ea{y^k H_1(\langle \vec{v}, \vec{z} \rangle)} = \langle \vec{v}, C_1^{(k)} \rangle \]
    As a result, if $\vec{v} = \sum_{k \geq 0} v_k C_1^{(k)}$, then by letting $p(y) = \sum_{k \geq 0} a_k y^k$ we have
    \[ \Ea{p(y) H_1(\langle \vec{v}, \vec{z} \rangle)} = \norm{\vec{v}}^2 \neq 0.\]
    Conversely, if $\vec{v}$ is orthogonal to every vector $C_1^{(k)}$, then $\Ea{y^k H_1(\langle \vec{v}, \vec{z} \rangle)} = 0$ for all $k \geq 0$, which implies that $\vec{v} \notin U^\star$.

    The case $\ell_p^\star = 2$ is similar, since this time
    \[ \Ea{y^k H_2(\langle \vec{v}, \vec{z} \rangle)} = \vec{v}^\top C_2^{(k)} \vec{v}. \]
    As a result, for fixed $k$, the largest vector space $U_k$ such that $\Ea{y^k H_2(\langle \vec{v}, \vec{z} \rangle)}$ for all $\vec{v}\in U_k$ is $\ker(C_2^{(k)}) = \mathrm{Im}(C_2^{(k)})^\bot$. The result ensues from $U^\star = \bigcap_{k \geq 1} U_k^\bot$.
\end{proof}

Using Assumption~\ref{assump:derivatives}, the following corollary ensues:
\begin{corollary}
    If $\ell_p^\star = 1$, then for any $\vec{v}\in U^\star$ and almost every $\rho_0$,
    \[ \langle \vec{v}, \psi(\vec{0}) \rangle \neq 0. \]
    As a result, with probability one, the set $\{ \psi(\vec{0}; \rho_{0, i}) \}_{i \in [p]}$ spans $U^\star$ as soon as $p \geq k$.
    
    If $\ell_p^\star = 2$, then with probability at least $1/2$ on the choice of $\rho_0, a_0$, $d\psi(\vec{0})$ is symmetric and has a strictly positive eigenvalue.
\end{corollary}

\begin{proof}
    Let $\vec{v}\in U^\star$, and $k \geq 1$ such that $\langle \vec{v}, C_1^{(k)} \rangle \neq 0$. Then 
    \[\left.\frac{\partial^k \langle \vec{v}, \psi(\vec{0}; \rho_0) \rangle}{\partial \rho_0^k}\right|_{\rho_0 = 0} = u_0^{(k)} \langle \vec{v}, C_1^{(k)} \rangle \neq 0, \]
    so the function $\rho_0 \mapsto \langle \vec{v}, \psi(\vec{0}; \rho_0) \rangle$ is not the zero function. Since it is also analytic, it is nonzero for almost every choice of $\rho$.
    
    We now prove that if $p \leq \dim(U^\star)$, then the family $\{ \psi(\vec{0}; \rho_{0, i}) \}_{i \in [p]}$ is free, which implies our second result. Assume by recursion that $\{ \psi(\vec{0}; \rho_{0, i}) \}_{i \in [p-1]}$ is free, and let $\vec{v}_p \in U^\star$ orthogonal to this family. Then with probability one, $\langle \vec{v}_p, \psi(\vec{0}; \rho_{0, p}) \rangle \neq 0$, so $\psi(\vec{0}; \rho_{0, p})$ is independent from $\{ \psi(\vec{0}; \rho_{0, i}) \}_{i \in [p-1]}$.

    Finally, if $\ell_p^\star = 2$, then by symmetry of the distribution of $\rho_0$ we can write $\psi(\vec{p}; a_0, \rho_0) = a_0 \psi(\vec{p}; 1, \rho')$ with $\rho' = a_0 \rho_0$ chosen independently from $a_0$. Writing
    \[ d\psi(\vec{0}; \rho') = \sum_{k \geq 1} (\rho')^k u_1^{(k)} C_2^{(k)}, \]
    we can see that $d\psi(\vec{0})$ is a symmetric matrix, which is non-zero with probability one on $\rho'$. Since either $d\psi(\vec{0})$ or $-d\psi(\vec{0})$ has a positive eigenvalue, $a_0d\psi(\vec{0})$ does with probability $1/2$.
\end{proof}

We are now ready to prove Theorem~\ref{thm:multi_index_recovery}. 
\begin{proof}[of Theorem~\ref{thm:multi_index_recovery}]
    We begin with the case $\ell_p^\star = 1$. Fix a $p = p(\delta)$ to be chosen later, and let $\psi_i(\vec{p}) = \psi(\vec{p}; a_{0, i}, \rho_{0, i})$ for $i \in [p]$. With probability $1/2$ on $\vec{w}_{0, i}$, we have $\langle \vec{w}_{0, i}, \psi_i(\vec{0}) \rangle$; let $p_+$ be the number of such neurons. For $p = O(\log(1/\delta))$, we have $p_+ > \max(p/4, k)$ with probability at least $1 - \frac\delta2$.

    In particular, with probability one, the set $\{ \psi_i(\vec{0}) \}_{i\in [p_+]}$ spans $U^\star$. Let $\epsilon > 0$ to be chosen later, $\eta_i(\epsilon)$ adapted to $\psi_i$ in Proposition~\ref{prop:app:hitting_times}, $\tau_i$ the hitting time for $\eta_i$ in the $i$-th neuron. Consider the time $\tau = \min_i \tau_i$; applying Proposition \ref{prop:app:hitting_times} to $\delta' = \delta/p$, we have
    \[ \tau \leq \max_{i \in [p_+]} \frac{\eta_i(1 + \sqrt{k}\epsilon)}{\gamma_0 \norm{\psi_i(\vec{0})}} d \leq C(\delta) d. \]
    However, we also have $\tau > c(\delta) \gamma^{-1}$, and hence from Proposition~\ref{prop:app:low_dim_linear}, for any $i \in [p_+]$
    \[ \norm{\vec{w}_{t, i}} \geq \gamma \tau \norm{\psi_i(\vec{0})} - \sqrt{k}\eta_i\epsilon \geq c'(\delta) \]
    for small enough $\epsilon$. This proves the first part of the theorem. Now, consider the vectors
    \[ \vec{u}_i = \frac{\Pi^\star \vec{w}_{\tau, i}}{\norm{\Pi^\star \vec{w}_{\tau, i}}}; \]
    by the same argument as in Proposition~\ref{prop:app:hitting_times}, for any $i \in [p_+]$,
    \[ \norm{\vec{u}_i - \frac{\psi_i(\vec{0})}{\norm{\psi_i(\vec{0})}}} \leq \sqrt{k}\epsilon. \]
    Since the $\{\psi_i(\vec{0})\}_{i\in[p_+]}$ span $U^\star$ with probability one, so do the $\{\vec{u}_i\}_{i \in [p_+]}$ for a small enough choice of $\epsilon$.

    For the case $\ell_p^\star = 2$, we define the $p_+$ ``good'' neurons as those with $d\psi_i(\vec{0})$ having a positive eigenvalue. Since this happens with probability at least $1/2$ independently for each neuron, the bounds on $p(\delta)$ from earlier still apply.
    For those neurons, and small enough $\epsilon$, we define the $\eta_i(\epsilon)$ as those in Proposition~\ref{prop:app:hitting_times}, and $\tau = \max_i \tau_i$. Then by Proposition~\ref{prop:app:hitting_times}, we have $\tau \leq C(\delta)d\log(d)^2$ for some constant $C(\delta)$.

    Now, fix a vector $i \in [p_+]$. We define the $j$-th ``excursion'' to be the $j$-th sequence of times $\tau_i^{(j)}, \dots, \tau_i^{(j)} + e_i^{(j)}$ such that $\norm{\Pi^\star\vec{w}_{t, i}} \geq \eta_i$. Then:
    \begin{itemize}
        \item  With probability at least $1 - e^{-c\gamma}$, we have $\norm{\vec{w}_{t+1, i} - \vec{w}_{t, i}} \leq C \sqrt{\gamma}$, and thus for any excursion $j$ we have $\norm{\Pi^\star\vec{w}_{\tau_i^{(j)} + e_i^{(j)} +1, i}} \geq \eta_i/2$.
        \item between two excursions, the results of Proposition~\ref{prop:app:low_dim_exp} apply, and thus the norm of $\norm{\Pi^\star \vec{w}_{t, i}}$ increases exponentially.
    \end{itemize}
   As a result, for $\tau_i \leq t \leq \tau$, we have $\norm{\Pi^\star \vec{w}_{t, i}} > \eta_i / 2$, which finishes the proof of Theorem~\ref{thm:multi_index_recovery}.
\end{proof}

\section{Additional proofs}

\subsection{Proof of Lemma~\ref{lem:bias_activation}}

Let $\sigma$ be an analytic function which is not a polynomial. Then for any $k \geq 0$, the function $\sigma^{(k)}$ is also analytic and non-identically zero, and hence so is the function $f_n = \sigma^{(k)} \cdot (\sigma')^k$. Lemma \ref{lem:bias_activation} thus ensues from the following result:
\begin{lemma}
  Let $f$ be an analytic and non-identically zero function. Then
  \[ \Ea{f(x+b)} \neq 0 \quad \text{and} \quad \Ea{x f(x+b)} \neq 0 \]
  for almost every $b$ under the Lebesgue measure.
\end{lemma}
\begin{proof}
  The function $\psi(b) = \Ea{f(x+b)}$ is analytic in $b$, and we have
  \[ \psi^{(k)}(0) = \Ea{f^{(k)}(x)} = c_k u_k(f), \]
  where $u_k(f)$ is the $k$-th Hermite coefficient of $f$ and $c_k$ is a nonzero absolute constant. Since $f$ is nonzero, at least one of its Hermite coefficients is nonzero, and $\psi$ is a non-identically zero analytic function. As a result, $\psi(b) \neq 0$ for almost every choice of $b$.

  The other case is handled as the first one, noting that by Stein's lemma
  \[ \Ea{x f(x+b)} = \Ea{f'(x+b)}. \]
\end{proof}
The above shows that for fixed $n \in \mathbb{N}$, the set
\[ \mathcal{B}_n = \left\{ b \in \R \ : \ \Ea{\sigma^{(n)}(x+b) \sigma'(x+b)^{n-1}} = 0 \quad \text{or} \quad \Ea{x \sigma^{(n)}(x+b) \sigma'(x+b)^{n-1}} = 0  \right\} \]
has measure $0$. Since there is a countable number of such subsets, the set
\[ \mathcal{B} = \bigcup_{n \in \mathbb{N}} \mathcal{B}_n \]
also has measure zero, which is equivalent to the statement of Lemma~\ref{lem:bias_activation}.

\subsection{Proof of Theorem~\ref{thm:poly_pge}}

We decompose the polynomial $h^\star$ as an even and odd part $e(x)$ and $o(x)$, with degrees $d_o$ and $d_e$. We first assume that $o$ is non-zero, and that the leading coefficient of $o$ is positive. Then, for odd $m \geq 0$, we have
\[ \Ea{xh^\star(x)^m} = \sum_{k = 0}^m \dbinom{m}{k} \Ea{x e(x)^k o(x)^{m-k}} \]
Each term in the above sum is zero if $k$ is odd, we focus on the terms where $k$ is even. There exist constants $A, \varepsilon > 0$ such that if $|x| \geq A$,
\[ |e(x)| \geq \varepsilon |x|^{d_e} \quad \text{and} \quad |o(x)| \geq \varepsilon |x|^{d_o}, \]
and we let
\[ B = \sup_{x \in  [-A, A]} |e(x)| \vee \sup_{x \in  [-A, A]} |o(x)| \]
As a result, when $k$ is even, both $e(x)^k$ and $xo(x)^{m-k}$ are even polynomials with positive leading coefficients, so
\begin{align*}
  \Ea{x e(x)^k o(x)^{m-k}} &= \Ea{x e(x)^k o(x)^{m-k} \mathbf{1}_{x\in [-A, A]}} + \Ea{x e(x)^k o(x)^{m-k} \mathbf{1}_{x\notin [-A, A]}} \\
  &\geq \varepsilon^2\Ea{ x^{kd_e + (m-k)d_o+1} \mathbf{1}_{x\notin [-A, A]}} - AB^m
\end{align*}
Let $d(m, k) = kd_e + (m-k)d_o+1$. Then,
\begin{align*}
  \Ea{x e(x)^k o(x)^{m-k}} \geq \varepsilon^2\Ea{x^{d(m, k)}} -  A^{d(m, k)} - AB^m.
\end{align*}
Going back to the sum, and with the crude bound $\dbinom{m}{k} \leq 2^m$,
\begin{align*}
  \Ea{xh^\star(x)^m} &\geq \varepsilon^2\Ea{x^{d(m, 0)}} - \sum_{k=1}^m 2^m\left(A^{d(m, k)} + AB^m\right) \\
  &\geq \Ea{x^{m+1}} - C^m
\end{align*}
Since the Gaussian moments grow faster than any power of $m$, this last expression is non-zero for a large enough choice of $m$.

Finally, if $o$ is the zero polynomial, we can do the same reasoning with $e(x)^m$ to get
\[ \Ea{(x^2-1) e(x)^m} = \varepsilon^2(\Ea{x^{m d_e+2}} - \Ea{x^{m d_e}}) - C^m = \varepsilon^2 md_e \Ea{x^{m d_e}} - C^m,\]
and we conclude as before.

\subsection{Hardness of sign functions}
We show in the appendix the following statement:
\begin{proposition}\label{prop:app:sign_hard}
  Let $m \in \mathbb{N}$, and
  \[ h^\star(\vec{x}) = \mathrm{sign}(x_1\dots x_m) \]
  Then the function $h^\star$ has Information and Generative Exponents $\ell = \ell^\star = m$.
\end{proposition}

We first show the lower bound. For any vector $\vec{v}\in \R^m$, and $k < m$, $H_k(\langle \vec{v}, \vec{x} \rangle)$ is a polynomial in $z$ of degree $k$, and therefore each monomial term is missing at least one variable. We show that this implies the lower bound of Proposition~\ref{prop:app:sign_hard} through the following lemma:
\begin{lemma}
  Let $f: \{-1, 1\} \to \R$ be an arbitrary function, and $g: \R^m \to \R$ a function which is independent from $x_m$. Then
  \[ \Ea{f(h^\star(\vec{x})) g(\vec{x})} = \frac{f(1)+f(-1)}{2}\Ea{g(\vec{x})}\]
\end{lemma}
\begin{proof}
  We simply write
  \begin{align*}
    \Ea{f(h^\star(\vec{x})) g(\vec{x})}
    &= \Ea{\Ea{f(h^\star(\vec{x})) g(\vec{x}) \mid x_1, \dots, x_{m-1}}}\\
    &= \Ea{g(\vec{x}) \Ea{f(h^\star(\vec{x})) \mid x_1, \dots, x_{m-1}}} \\
    &= \Ea{g(\vec{x}) \frac{f(1)+f(-1)}{2}} \\
    &=  \frac{f(1)+f(-1)}{2}\Ea{g(\vec{x})},
  \end{align*}
  where at the second line we used that $g$ only depends on $x_1, \dots, x_{m-1}$ and at the third line the addition of $x_m$ randomly flips the sign of $h^\star$.
\end{proof}

Using this property on every monomial in $H_k(\langle \vec{v}, \vec{x} \rangle)$ and summing back, we have for any function $f$
\begin{equation}
  \Ea{f(h^\star(x)) H_k(\langle \vec{v}, \vec{x} \rangle)} = \frac{f(1)+f(-1)}{2}\Ea{H_k(\langle \vec{v}, \vec{x} \rangle)}
\end{equation}
But the latter is zero when $k\geq 1$ by orthogonality of the Hermite polynomials. This shows that $\ell_{\vec{v}}^\star \geq m$ for all $\vec{v}\in\R^k$, and therefore $m \leq \ell^\star \leq \ell$.

For the upper bound, let $\vec{v} = \vec{e}_1 + \dots + \vec{e_m}$. Then
\[ H_k(\langle \vec{v}, \vec{x} \rangle) = m! \, x_1\dots x_m + Q(\vec{x}), \]
where $Q$ is a polynomial where each monomial has a missing variable. As a result,
\begin{align*}
  \Ea{h^\star(\vec{x}) H_k(\langle \vec{v}, \vec{x} \rangle)} &= m!\, \Ea{h^\star(\vec{x}) x_1\dots x_m} \\
  &= m!\, \Ea{|x_1\dots x_m|} \\
  &= m! \left( \sqrt{\frac2\pi} \right)^m
\end{align*}
which is a non-zero value, hence $\ell^\star \leq \ell \leq m$.

\section{Additional Numerical Investigation}
\label{sec:app:numerics}
In this section, we present further numerical investigation to support our theory and extend it beyond the mathematical hypotheses used in the formal proof. 

\subsection{Testing wider range of Algorithm~\ref{algo:optimizer_main_step}}
In the main we presented all the numerical simulation using \(\rho>0\), i.e. \emph{ExtraGradient method/descent}. 
In Figure~\ref{fig:app:other-algos}~Left, we repeat one of the experiment of Figure~\ref{fig:main:single_index_complete} in the case \(\rho<0\), the \emph{Sharpe Aware Minimization}. The two algorithms are based on very different principle, but in our context they behave in the same way: the reusing of the data makes high information exponent function easily learnable. The global picture for \(\rho\neq0\) is the same regardless of its sign. 

\begin{figure}
    \centering
    \includegraphics[width=0.4\textwidth]{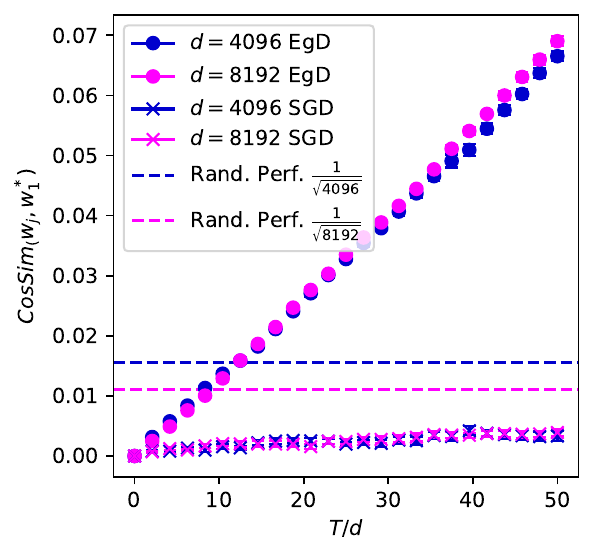}
    \includegraphics[width=0.4\textwidth]{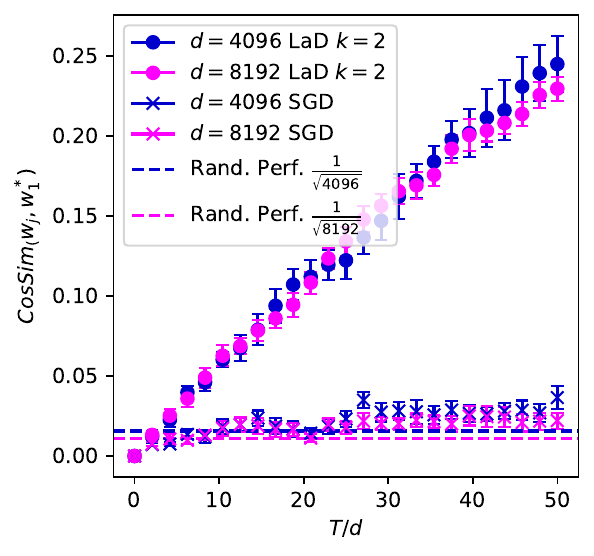}
    \caption{
        example of two different algorithm with data repetitions learning an hard single-index target \(h^\star(\vec z)=\mathrm{He}_3(z_1)\), \(\ell =3, \ell^\star=1\).
        \textbf{Left}: SAM, with \(\rho_0=0.1,\gamma_0=0.01\).
        \textbf{Right}: 2-Lookahed with \(\gamma_0=0.1\).
        See details in App.~\ref{sec:app:implementation}.
    }
    \label{fig:app:other-algos}
\end{figure}
Altought Algorithm~\ref{algo:optimizer_main_step} includes a broad variety of algorithm, it is not the most general domain of validity of our theory. Simple algorithms such as repeating the gradient step twice on the same data before discarding it (also known as \emph{2-lookahead optimizer} \citep{lookahead}) are not part of Algorithm~\ref{algo:optimizer_main_step}, yet our considerations are still valid. In fact, the data repetition is sufficient to introduce correlations across two consecutive steps, whose effect builds up along the dynamics causing the network to learn hard targets. A simple illustration of this claim is available in Figure~\ref{fig:app:other-algos}~Right.

\subsection{Extensive batch size}
In this section, we want to present some evidence that our result does not change when the batch size is $1 < n_b \le O(d^{\frac{\ell^\star}{2}})$. The formal proof we presented in the paper is valid in the case where the batch size is \(n_b=1\), but we strongly believe it is the case when training with larger batches.

In the context of the information exponent, it has been shown in \cite{arnaboldi2024online} that the total sample complexity does not change when using batch sizes \(n_b \le O\left(d^{\frac{\ell}{2}}\right)\). The heuristic argument for this behavior is that using a larger batch size increases the learning rate since the noise of the gradient is ``more'' averaged out, ultimately reducing the number of steps.
The same argument should pass from information exponent to generative exponent, allowing the claim (when \(n_b \le O\left(d^{\frac{\ell^\star}{2}}\right)\)) 
\begin{align}
    T \cdot n_b \sim 
        \begin{cases}
        &O(d^{\ell -1} ) \qquad \hspace{0.8em}\text{if $\ell^\star>2$} \\
        &O(d \log{d} ) \qquad \text{if $\ell^\star=2$} \\
        &O(d) \qquad \hspace{2.2em}\text{if $\ell^\star=1$}
    \end{cases}.
\end{align}

Note that this drastically reduces the time steps needed to learn the teacher's directions, although the sample complexity does not change. \cite{dandi2024benefits} first shows that some high-information exponent functions, such as \(\mathrm{He}_3\), can be learned in just a few steps with full-batch gradient descent \(n_b = O(d)\). In the same work, they introduce a larger class of hard functions that cannot be recovered in \(T=O(1)\), not providing any information on possible bounds on the number of time steps, nor any claim suggesting that not all the function in the class have the same sample complexity. Our theory can provide a full understanding of the phenomena:
\begin{itemize}
    \item \(\mathrm{He}_3\) has generative exponent \(\ell^\star=1\), that implies we need \(O(d)\) samples to weakly recover the target; when training with batch size \(n_b=O(d)\), we have \(T=O(1)\) as in \cite{dandi2024benefits}; Figure~\ref{fig:app:single-index-large-batch}~(Left) is a numerical proof of this fact.
    \item \(\mathrm{He}_4\) has generative exponent \(\ell^\star=2\), that implies we need \(O(d\log d)\) samples to weakly recover the target; when training with batch size \(n_b=O(d)\), we have \(T=O(\log d)\); Figure~\ref{fig:app:single-index-large-batch}~(Right) is a numerical proof of this fact.
\end{itemize}

\begin{figure}
    \centering
    \includegraphics[height=0.4\textwidth]{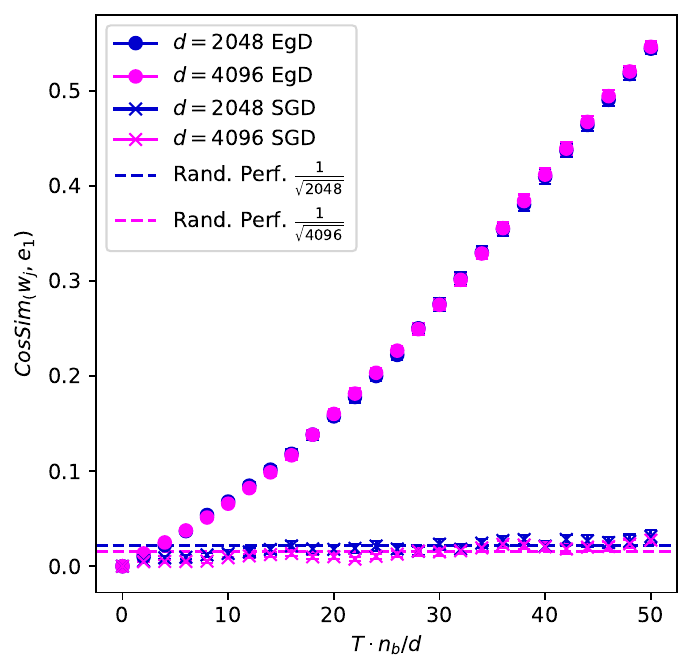}
    \includegraphics[height=0.4\textwidth]{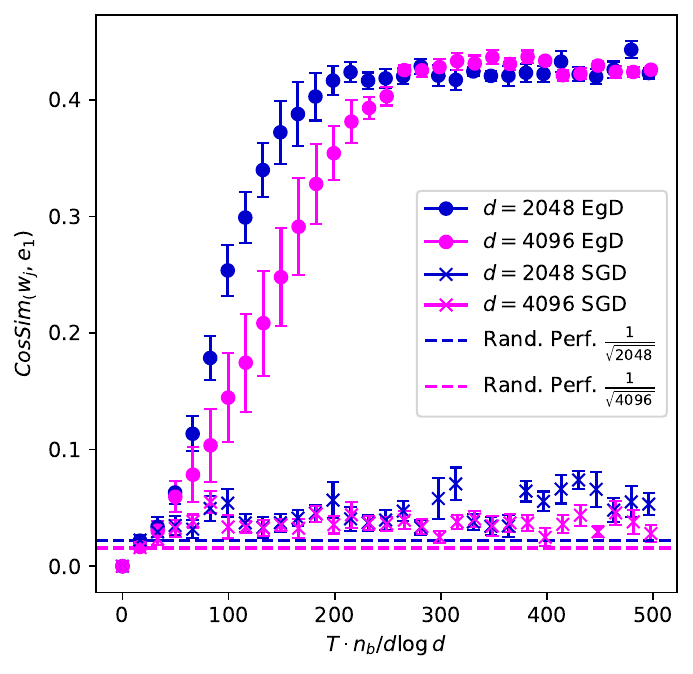}
    \caption{
        Single-index recovery for large-batch training.
        The dashed horizontal line $\frac{1}{\sqrt{d}}$ is a visual guide to place random performance.
        The plots show $\sigma = \mathrm{relu}$, $\gamma=0.01$, $\rho = 0.1$, and the performance is computed averaging across 10 different runs.
        \textbf{Left $(\ell,\ell^\star)=(3,1)$:} $h^\star = \mathrm{He}_3.$
        \textbf{Right $(\ell,\ell^\star)=(4,2)$:} $h^\star = \mathrm{He}_4$.
        See details in App.~\ref{sec:app:implementation}.
    }
    \label{fig:app:single-index-large-batch}
\end{figure}

We can push our claims beyond single-index targets and illustrate the same phenomena for multi-index functions.
In Figure~\ref{fig:app:multi-index-large-batch} we test Extragradient Descent against two hard multi-index functions: once again the sample complexity is the same as the case \(n_b=1\), while the number of steps is reduced by a factor \(n_b=d\).

\begin{figure}
    \centering
    \includegraphics[height=0.45\textwidth]{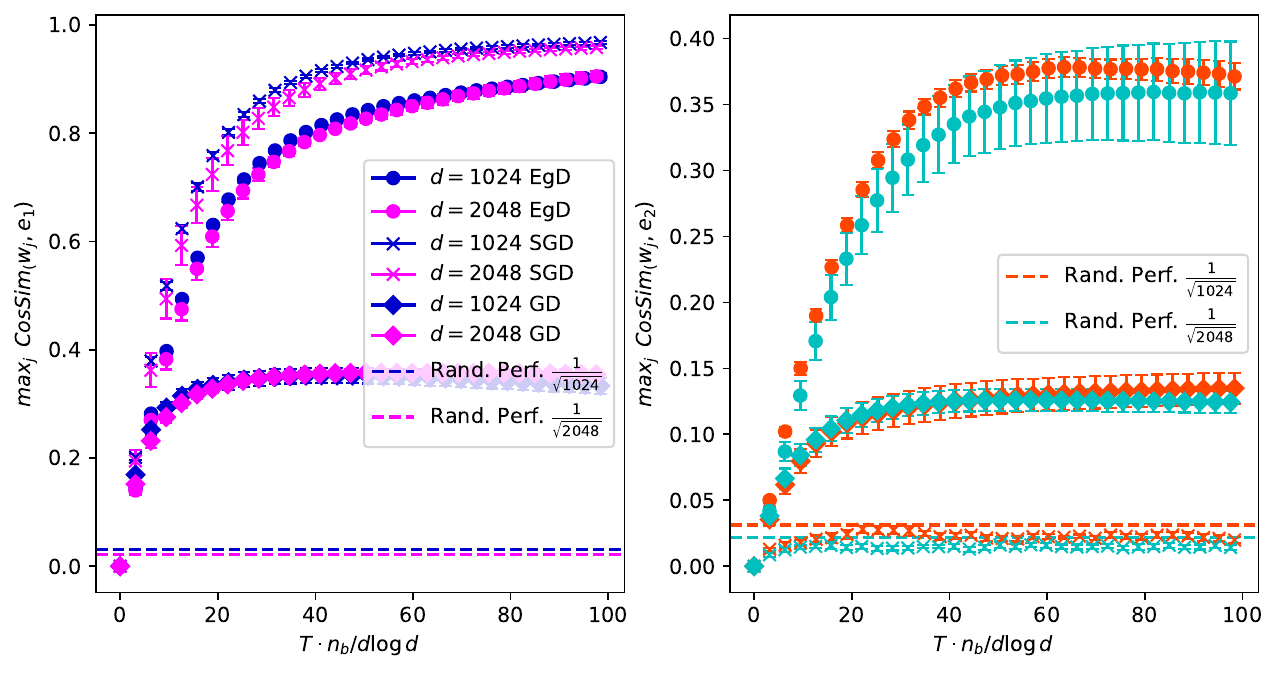}
    \includegraphics[height=0.33\textwidth]{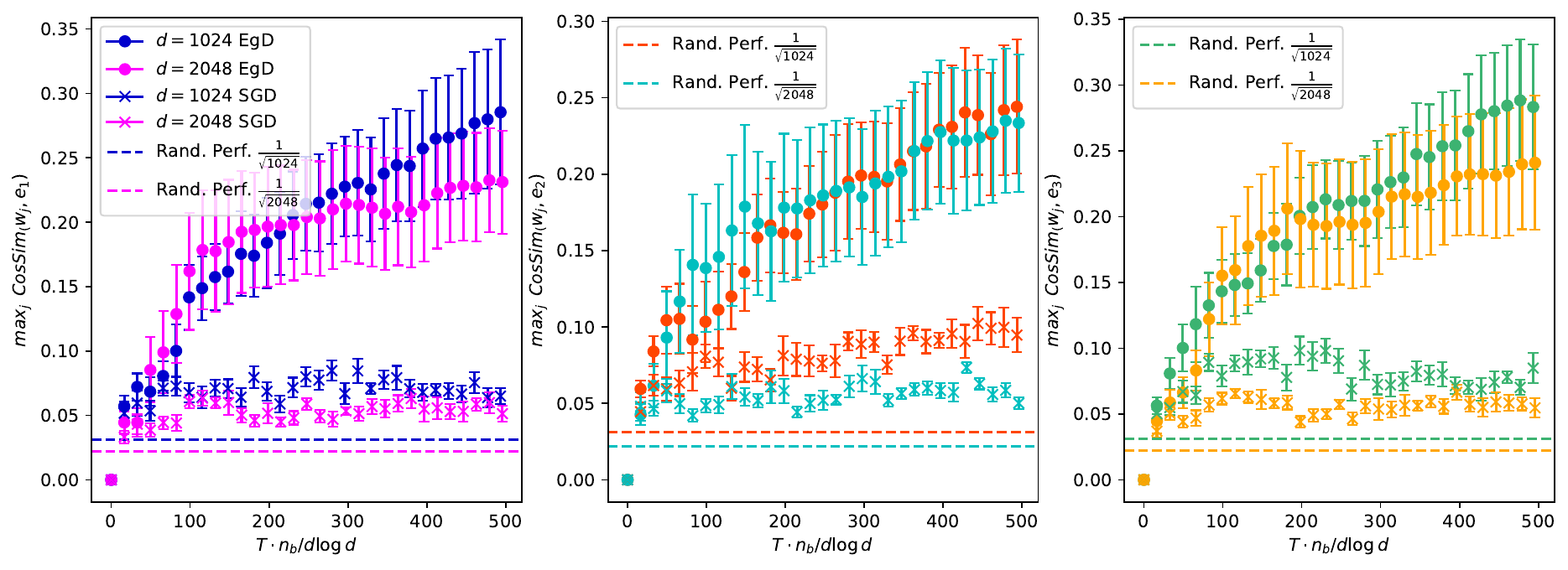}
    \caption{%
        Multi-index recovery for large-batch training.
        The dashed horizontal line $\frac{1}{\sqrt{d}}$ is a visual guide to place random performance.
        The plots show $\sigma = \operatorname{relu}$.
        \textbf{Upper $(\ell_{\vec{v}},\ell_{\vec{v}}^\star)=(3,1)$:} $h^\star(\vec z) = \operatorname{sign}(z_1z_2).$
        \textbf{Lower $(\ell_{\vec{v}},\ell_{\vec{v}}^\star)=(4,2)$:} $h^\star(\vec z) = z_1z_2z_3$.
        See Appendix~\ref{sec:app:implementation} for details.}
    \label{fig:app:multi-index-large-batch}
\end{figure}

\section{Implementation Details} \label{sec:app:implementation}
In this section. we provide all the implementation details needed to reproduce the Figures of the paper. The full detailed implementation can be repository linked in the main paper.

Unless explicitly stated, we always run the plain version of Algorithm~\ref{algo:optimizer_main_step} with squared loss, as opposed to Assumption~\ref{assump:spherical_corr} which requires \emph{correlation loss} and spherical gradient. We use the reLU function as an activation. Finally, we always run simulations with \(W^0\bot W^\star\), which does not fall under Assumption~\ref{assump:init}. Indeed, our numerical simulations deviate from the theoretical assumptions needed for the formal proofs, but the aim here is to show that our claims are valid beyond the technical limitations of the theory and that they can allow us to understand settings closer to what is used in practice.

In the plots with a single-index target we plot the \emph{absolute value} of the cosine similarity between the student weight \(\vec{w}\) and the teacher weight \(\vec{w}^\star\), averaged across \(N\) different runs (where we vary both seen samples and initial conditions). If \(q\) in the index for different runs:
\[\text{yaxis}(t)=\frac1N\sum_{q=1}^N\left|\mathrm{CosSim}(\vec w^{(q)}_t,\vec w^{\star{(q)}})\right|,\]
where
\[\mathrm{CosSim}(\vec w,\vec w^{\star}) = \frac{\vec w\cdot\vec w^\star}{\norm{\vec w}_2\norm{\vec{w}^\star}_2}.\]
In the plots with a multi-index target, we take the maximum across all the network weights
\[\text{yaxis}(t)=\frac1N\sum_{q=1}^N\max_{j\in[p]}\left|\mathrm{CosSim}(\vec w^{(q)}_{j,t},\vec w^{\star{(q)}})\right|.\]
\subsection{Hyperparameters of the Figures}
Here are the hyperparameters used in the figures:
\begin{itemize}
    \item \textbf{Figure~\ref{fig:main:single_index_complete} (all 4)}: \(d=8192, \gamma_0=0.01, \rho_0 = 0.1, N=40\).
    \item \textbf{Figure~\ref{fig:main:multi_index_complete} (left)}: \(p=8, d=2048, \gamma_0=0.01, \rho_0 = 0.1, N=40\), \(a^0_j\sim \mathrm{Rad}(1/2)\).
    \item \textbf{Figure~\ref{fig:main:multi_index_complete} (center-left)}: \(p=8, d=2048, \gamma_0=0.01, \rho_0 = 0.1, N=40\), \(a^0_j\sim \mathcal{N}(0,1)\).
    \item \textbf{Figure~\ref{fig:main:multi_index_complete} (center-right)}: \(p=8, d=2048, \gamma_0=0.1, \rho_0 = 0.1, N=40\), \(a^0_j\sim \mathcal{N}(0,1)\).
    \item \textbf{Figure~\ref{fig:main:multi_index_complete} (right)}: \(p=8, d=2048, \gamma_0=0.1, \rho_0 = 0.1, N=20\), \(a^0_j\sim \mathcal{N}(0,1)\).
    \item \textbf{Figure~\ref{fig:main:staircase1} (left)}: \(p=8, d=2048, \gamma_0=0.1, \rho_0 = 0.1, N=40\), \(a^0_j\sim \mathcal{N}(0,1)\).
    \item \textbf{Figure~\ref{fig:main:staircase1} (right)}: \(p=4, d=1024, \gamma_0=0.01, \rho_0 = 0.1, N=10\), \(a^0_j\sim \mathcal{N}(0,1)\).
     \item \textbf{Figure~\ref{fig:main:staircase2}}: \(p=4, d=1024, \gamma_0=0.01, \rho_0 = 0.1, N=10\), \(a^0_j\sim \mathcal{N}(0,1)\).
    
    \item \textbf{Figure~\ref{fig:app:other-algos} (left)}: SAM \(d=4096,8192, \gamma_0=0.01, \rho_0 = 0.1, N=40\).
    \item \textbf{Figure~\ref{fig:app:other-algos} (right)}: 2-Lookahead \(d=4096,8192, \gamma_0=0.01, \rho_0 = 0.1, N=40\).
    \item \textbf{Figure~\ref{fig:app:single-index-large-batch} (all 2)}: \(d=8192,  n_b=2d, \gamma_0=0.1, N=40\).
    \item \textbf{Figure~\ref{fig:app:multi-index-large-batch} (upper)}: \(p=8, d=2048, n_b=2d, \gamma_0=0.1, \rho_0 = 0.1, N=40\), \(a^0_j\sim \mathcal{N}(0,1)\).
    \item \textbf{Figure~\ref{fig:app:multi-index-large-batch} (lower)}: \(p=8, d=2048, n_b=2d, \gamma_0=0.1, \rho_0 = 0.1, N=20\), \(a^0_j\sim \mathcal{N}(0,1)\).
\end{itemize}

\end{document}